\documentclass{article}
\usepackage{graphicx,fullpage} 

\usepackage[utf8]{inputenc} 
\usepackage[T1]{fontenc}    
\usepackage{hyperref}       
\usepackage{url}            
\usepackage{booktabs}       
\usepackage{amsfonts}       
\usepackage{nicefrac}       
\usepackage{microtype}      
\usepackage{amsmath}
\usepackage{multirow}
\usepackage{amsthm}
\usepackage{graphicx}
\usepackage{xparse}
\hypersetup{
    colorlinks=true,
    linkcolor=Blue,
    filecolor=magenta,      
    urlcolor=Cerulean,
    citecolor=blue,
    pdftitle={Offline Multi-task Transfer RL with Representational Penalization},
    pdfpagemode=FullScreen,
    }

\usepackage{wrapfig}
\usepackage[dvipsnames]{xcolor}
\usepackage{subcaption}
\usepackage{algorithm}
\usepackage{algorithmic}

\newtheorem{remark}{Remark}
\newtheorem{lemma}{Lemma}
\newtheorem{assumption}{Assumption}

\newtheorem{definition}{Definition}[section]
\usepackage{enumitem}
\usepackage{thmtools,thm-restate}

\DeclareMathOperator*{\amax}{\mathrm{argmax}}

\usepackage[round, sort&compress]{natbib} 
\title{Offline Multi-task Transfer RL with Representational Penalization}
\author{
Avinandan Bose\\ University of Washington\\
avibose@cs.washington.edu \and Simon Shaolei Du \\ University of Washington \\ ssdu@cs.washington.edu \and Maryam Fazel\\University of Washington\\mfazel@uw.edu 
}
\date{}

\begin{document}

\maketitle

\begin{abstract}

We study the problem of representation transfer 
in offline Reinforcement Learning (RL), where a learner has access to episodic data from a number of source tasks collected a priori, and aims to learn a shared representation to be used in finding a good policy for a target task. Unlike in {online} RL where the agent interacts with the environment while learning a policy, in the \emph{offline} setting there cannot be such interactions in either the source tasks or the target task; thus multi-task offline RL can suffer from incomplete coverage.

We propose an algorithm to compute pointwise uncertainty measures for the learnt representation, and establish a data-dependent upper bound for the suboptimality of the learnt policy for the target task. Our algorithm leverages the collective exploration done by source tasks to mitigate poor coverage at some points by a few tasks, thus overcoming the limitation of needing uniformly good coverage for a meaningful transfer by existing offline algorithms.
We complement our theoretical results with empirical evaluation on a rich-observation MDP which requires many samples for complete coverage. 
Our findings illustrate the benefits of penalizing and quantifying the uncertainty in the learnt representation.
\end{abstract}


\section{Introduction}
The ability to leverage historical experiences from past tasks and transfer the shared skills to learn a new task with only a few interactions with the environment is a key aspect of machine intelligence. In this paper, we study this goal in the context of multi-task reinforcement learning (MTRL). 
 
Multi-task learning has been widely studied across different paradigms.
\citet{caruana1997multitask,pan2009survey}
study a transfer learning scenario where the learner is equipped with data from various source tasks during a pre-training phase. The objective is to learn features easily adaptable to a designated target task.
Similar problems are also studied in meta-learning \citep{finn2017model}, lifelong learning \citep{parisi2019continual} and curriculum learning \citep{liu2021curriculum}. The effectiveness of representation transfer for RL has also been studied in \cite{xu2020knowledge, zhang2022discovering, mitchell2021offline, kumar2022pre}.

Notably, in all these applications, task datasets are available to the learner a priori.
On the theoretical side, there has been a recent surge in emphasis on representation learning questions, driven by their practical significance in both supervised learning and reinforcement learning (RL). While the results in the supervised learning setup \citep{du2020few, tripuraneni2021provable, sun2021towards} 
can work in the offline setting with the assumption that data was collected independently and identically from the underlying distributions, in RL data collection is tied to the deployed policy. 
The main focus has been on the online setting where the learner is able to interact with source tasks to construct 
datasets with good ``coverage" by exploring extensively. Several recent papers study reward-free representation transfer learning \citep{jin2020reward, zhang2020task, wang2020reward, misra2020kinematic, agarwal2020flambe, modi2021model, agarwal2023provable}. These approaches are well-suited for scenarios with efficient data generators, such as game engines \citep{bellemare2013arcade} and physics simulators \citep{todorov2012mujoco}, serving as environments.


Online RL is harder in safety-critical domains, like precision medicine \citep{gottesman2019guidelines} and autonomous driving \citep{shalev2016safe}, where interactive data collection processes can be costly and risky. 
Offline datasets are often abundantly available, e.g., electronic health records for precision medicine \citep{chakraborty2014dynamic} and human driving trajectories for autonomous driving \citep{sun2020scalability}. However, guarantees for current algorithms in offline RL exist under restrictive assumptions as discussed in  \citep{levine2020offline, lange2012batch, wang2020statistical}, which often don't hold true on existing datasets.


In this paper we wish to ask the following question: 

\textit{
Can we design a provably sample efficient algorithm for offline  MTRL
under minimal assumptions on the datasets?}

We answer this question in the affirmative by introducing a novel algorithm that we provide a theoretical analysis for.
We list our contributions below.

\begin{enumerate} 
    \item We address a bottleneck in offline MTRL for low-rank MDPs (Definition~\ref{def:lowrankmdpdefinition}) 
    by 
    quantifying data-dependent pointwise uncertainty while trying to model the target task transition dynamics with the representation learnt from source tasks (cf.\ Theorem~\ref{thm:representationbound}). Quantifying pointwise uncertainty has not been addressed before even in single-task offline RL in low-rank MDPs due to non-linear function approximation \citep{uehara2021representation}, hence our techniques are of independent interest even for single-task offline RL. 
    \item Inspired by ideas in non-parametric estimation, we introduce a quantity termed effective occupancy density (cf.\ Algorithm~\ref{alg:pointwiseerror}) which captures the coverage of a certain state-action pair across all source datasets. We show that representation transfer error scales inversely with the square root of the effective occupancy density (cf.\ Theorem~\ref{thm:representationbound}). Our results show that extensively exploring \textit{every state-action pair} for \textit{every source task}, is \textit{not necessary} for uniformly low error for the representation transfer. In fact failure to explore certain state-action pairs by some task can be balanced out by the exploration done by other tasks (cf.\ Corollary~\ref{cor:condition_for_good_rep_transfer}). 
    \item We derive a data-dependent bound on the suboptimality of the learnt policy for the target task 
    (cf.\ Theorem~\ref{thm:suboptbound}). 
    highlighting three key factors affecting the success of the process (i) source tasks’ coverage of target task’s optimal policy, (ii) source tasks’ coverage of the offline samples from the target task, (iii) target task’s coverage of its optimal policy.
    \item We show that under mild assumptions on the policy collecting the data, the learner can achieve a near-optimal target policy by constructing source datasets of size polynomial in the covering number of the low-rank representation space and target dataset only polynomial in the dimension of the representation. This allows leveraging typically available vast historical data from several source tasks and then performing few shot learning for the target task (cf. Corollary~\ref{cor:sample_complexity}).
    \item We empirically validate our algorithm on the 
    benchmark in \citep{misra2020kinematic}, and demonstrate that popular approaches without penalising the representation transfer end up with suboptimal cumulative rewards.
\end{enumerate}


\subsection{Related Work}
\noindent \textbf{Online Multi-Task Transfer Learning in Low-rank MDPs:} Our setup is similar to that studied in \citep{cheng2022provable,lu2022provable,agarwal2023provable} which learns a representation from the source tasks and then uses the learnt representation to learn a good policy in the target task, where all tasks are modeled as low-rank MDPs. However, all consider the online setting where the learner uses reward-free exploration in the source tasks to construct datasets with good coverage. As mentioned earlier, this can be costly or risky in applications such as precision medicine or autonomous driving, which preferably rely on  offline data.

\noindent \textbf{Single Task Offline RL:} The main challenge in offline RL is insufficient dataset coverage, leading to distribution shift between trajectories in the dataset and those induced by the optimal policy \citep{wang2020statistical,levine2020offline}. This issue is prevalent safety-critical domains, like precision medicine \citep{gottesman2019guidelines}, autonomous driving \citep{shalev2016safe} and ride-sharing \citep{bose2021conditional}, where interactive data collection processes can be costly and risky. This issue is exacerbated by overparameterized function approximators, such as deep neural networks, causing extrapolation errors on less covered states and actions \citep{fujimoto2019off}. Theoretical study of offline RL typically requires one of these assumptions (i) the ratio between the visitation measure of the optimal policy and that of the data collecting
policy 
to be upper bounded uniformly over the state-action space \citep{jiang2016doubly, thomas2016data, farajtabar2018more, liu2018breaking, xie2019towards, nachum2019dualdice, nachum2020reinforcement, tang2019doubly, kallus2022efficiently, jiang2020minimax, uehara2021representation, du2019provably, yin2020asymptotically, yin2021near, yang2020off, zhang2020gendice} or (ii) the
concentrability coefficient defined as the supremum of a
similarly defined ratio over the state-action space needs to be upper bounded \citep{antos2007fitted,  munos2008finite, scherrer2015approximate, chen2019information, liu2019neural, wang2019neural, fan2020theoretical, xie2020q, liao2022batch, zhang2020gendice}. 

Recent algorithms proposed in \citep{yu2020mopo, kidambi2020morel, kumar2020conservative, liu2020provably, buckman2020importance, jin2021pessimism} provably work without any coverage assumptions by penalizing the exploration in offline datasets. The work closest to ours is \citep{jin2021pessimism} who bound the suboptimality of the learnt policy in terms of an uncertainty quantifier for the limited exploration. For a special instance of low-rank MDPs where the representation is assumed to be known (linear MDP \citep{jin2020provably}), \citep{jin2021pessimism} algorithmically construct an uncertainty quantifier. Our setup has the additional challenge of estimating the \textit{unknown} representation, and bounding the suboptimality of the learnt policy in terms of insufficient coverage in the datasets used to learn the representation as well as the dataset used to learn the policy. As discussed in \citep{uehara2021representation}, the non-linear function approximation in Low-rank MDPs as opposed to linear MDPs makes the uncertainty quantification very challenging.  Thus our techniques are of independent interest even for the single task offline RL in low-rank MDPs.
\section{Preliminaries}
In this paper, we study transfer learning in finite-horizon episodic Markov Decision Processes (MDPs), $\mathcal{M} = \langle H, \mathcal{S}, \mathcal{A}, \{P_h\}_{1:H}, \{r_h\}_{1:H}, d_1\rangle$, specified by the episode length or horizon $H$, state space $\mathcal{S}$, action space $\mathcal{A}$, \textit{unknown} transition dynamics $P_h : \mathcal{S} \times \mathcal{A} \rightarrow \mathcal{S}$, \textit{known} reward function $r_h : \mathcal{S} \times \mathcal{A} \rightarrow [0,1]$ and a \textit{known} initial state distribution $d_1$. For any Markov policy $\pi : \mathcal{S} \rightarrow \mathcal{A}$, we use the shorthand notation $\mathbb{E}_{P,\pi}$ to denote the expectation under the distribution of the trajectory induced by executing the policy $\pi$ in an MDP with transition dynamics $P = \{P_h\}_{1:H}$, i.e., start at an initial state $s_1 \sim d_1$, then for all $h \in [H]$, $a_h \sim \pi_h(s_h)$, $s_{h+1} \sim P_h(\cdot| s_h, a_h)$. The value function is the expected reward of a policy $\pi$ starting at state $s$ in step $h$, i.e., $V^{\pi}_{P, r; h}(s) = \mathbb{E}_{P,\pi}[\sum_{\tau=h}^H r_{\tau}(s_{\tau}, a_{\tau}) | s_h = s]$. The $Q$-function is $Q^{\pi}_{P,r;h}(s, a) = r_h(s, a) + \mathbb{E}_{s' \sim P_h(\cdot|s,a)}V^{\pi}_{P, r; h+1}(s')$. The expected total reward of a policy $\pi$ is defined as $V^{\pi}_{P, r} = \mathbb{E}_{s_1 \sim d_1}V^{\pi}_{P, r; 1}(s_1)$ and the optimal policy $\pi^\ast$ is denoted as the policy maximizing the expected total reward, i.e., $\pi^{\ast} = \amax_{\pi} V^{\pi}_{P, r}$. Our focus in this paper is on a special class of MDPs formalized below.
\begin{definition}\label{def:lowrankmdpdefinition}
    (Low-Rank MDP \citep{jiang2017contextual}, \citep{agarwal2020flambe}): A transition model $P_h : \mathcal{S} \times \mathcal{A} \rightarrow \mathcal{S}$ is low-rank with dimension $d$ if there exist two unknown embedding functions $\phi_h : \mathcal{S} \times \mathcal{A} \rightarrow \mathbb{R}^d$, $\mu_h : \mathcal{S} \rightarrow \mathbb{R}^d$ such that $\forall s,s' \in \mathcal{S}, a \in \mathcal{A} : P_h(s'|s,a) = \phi_h(s, a)^\top\mu_h(s')$, where $\|\phi_h(s, a)\|_2 \leq 1$ for all $(s, a)$ and for any function $g: \mathcal{S} \rightarrow [0, 1], \|\int g(s) \mu_h(s) \mathbf{d}(s)\|_2 \leq \sqrt{d}$. 
    An MDP is low rank if $P_h$ allows such a decomposition for all $h \in [H]$. 
\end{definition}

Low-rank MDPs capture several classes of MDPs such as the latent variable model \citep{agarwal2020flambe} where $\phi(s, a)$ is a distribution over
a discrete latent state space $\mathcal{Z}$, and the block-MDP model \citep{du2019provably} where $\phi(s, a)$ is a one-hot encoding vector. Note that since $\phi$ can be a non-linear, flexible function class, the low-rank framework generalizes prior works with linear representations \citep{hu2021near}, \citep{yang2020impact}, \citep{yang2022nearly}. 

The setup involves $K$ source tasks and one target task all of which can be modeled as low rank MDPs (see Figure~\ref{fig:schematic}). The learning process can be classified into 2 steps: (i) Representation Learning: The learner learns a shared representation 
across all the source tasks, (ii) Planning: 
With the learnt representation, the learner plans a good policy for the target task. We first list a few common  structural assumptions on the tasks which are needed for a meaningful representation transfer.

\begin{figure}
    \centering
    \includegraphics[width=\textwidth]{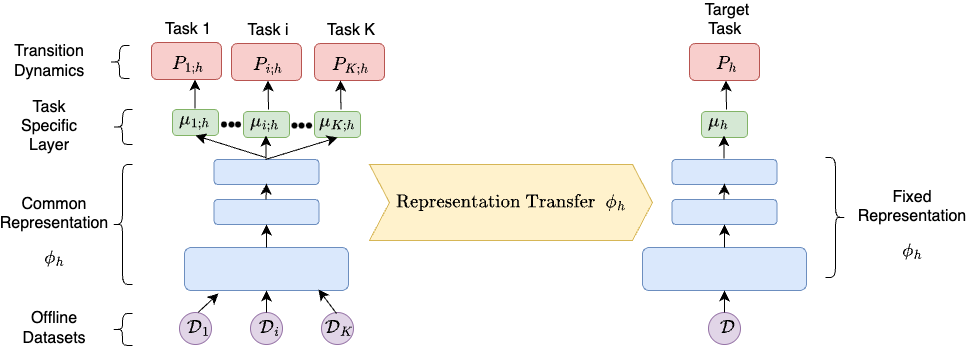}
    \caption{The learner has access to offline datasets from $K$ source tasks and one target task all of which are modelled as Low-rank MDPs. First a common representation is learned across all source tasks, and keeping this representation fixed, the learner plans a near optimal policy using the target task's dataset.}
    \label{fig:schematic}
\end{figure}

\begin{assumption}\label{ass:commonrep}
    (Common Representation): All tasks share a common representation $\phi^\star_h(s,a)$.
\end{assumption}
We denote the next state feature maps for the target task as $\mu^\ast_h$ and for the source tasks as $\{\mu^\ast_{1;h}, \ldots, \mu^\ast_{K;h}\}$.
\begin{assumption}\label{ass:linearspan}
(Pointwise Linear Span \citep{agarwal2023provable}) For any $h \in [H]$ and $s' \in \mathcal{S}$, there exists a vector $\alpha_{h}(s') \in \mathbb{R}^K$, such that $\mu^\ast_h(s') = \sum_{i \in [K]} \alpha_{i;h}(s')\mu^\ast_{i;h}(s')$, and $\alpha_{\mathrm{max}} = \max_{h;i;s' \in {[H] \times [K] \times [S]}} \alpha_{i;h}(s')$ is bounded.
\end{assumption}
These assumptions capture a large class of MDPs. \citep{cheng2022provable} study unknown source models with the same $\phi^\ast$ thus satisfying Assumption~\ref{ass:commonrep}. Assumption~\ref{ass:linearspan} is a strict generalization of mixture models where the target task transitions are a linear combination of the source tasks dynamics, studied by \citep{modi2020sample},\citep{ayoub2020model}, Block MDPs with shared latent dynamics \citep{du2019provably}. 

In this paper we consider tasks which can all be modeled as low rank MDPs satisfying Assumptions~\ref{ass:commonrep}, \ref{ass:linearspan}. We study the offline setting where a learner has access to datasets from $K$ source tasks, each containing $N_S$ episodic trajectories. 
Let the dataset corresponding to task $i$ be denoted as $\mathcal{D}_i = \{(s_h^{i; \tau}, a_h^{i; \tau})\}_{\tau, h=1}^{N_S, H}$. 
Since Assumption~\ref{ass:commonrep} states that all tasks share a common representation, our goal is to first learn a good estimate $\widehat{\phi_h}(s_h, a_h)$ of $\phi^\ast_h$ from the available offline data on source tasks and then do few shot offline training on a target task using this learned representation. The learner also has access to a dataset containing $n$ (typically $n \ll N_S$) episodic trajectories from the target task, denoted by $\mathcal{D} = \{(s_h^\tau, a_h^\tau)\}_{\tau, h = 1}^{n, H}$.  Our main goal is to learn a good policy $\pi$ for the target task using the learnt representation $\widehat{\phi}_h$ that maximizes the expected total reward. The performance metric is the suboptimality gap defined below.
\begin{definition}
    (Suboptimality Gap): The suboptimality gap for any given policy $\pi$ and initial state $s \in \mathcal{S}$ is the defined as
    \begin{align*}
        \textrm{SubOpt}(\pi, s) = V_1^{\pi^\star}(s) - V_1^{\pi}(s),
    \end{align*}
    where $\pi^\star$ is the optimal policy.
\end{definition}
In order to state the assumptions on the collected datasets, we begin with the following definition.
\begin{definition}\label{def:compliance}(Compliance \citep{jin2021pessimism}) For a dataset $\mathcal{D} = \{(s_h^\tau, a_h^\tau)\}_{\tau, h = 1}^{n, H}$, let $P_{\mathcal{D}}$ be the joint distribution of the data collecting process. We say $\mathcal{D}$ is compliant with the underlying MDP $\mathcal{M}$ if 
\begin{align*}
    P_{\mathcal{D}}(s_{h+1}^\tau = s' | \{(s_h^j, a_h^j)\}_{j=1}^\tau, \{s_{h+1}^j\}_{j=1}^{\tau - 1}) = P_h(s_{h+1} = s' | s_h = s_h^\tau, a_h = a_h^\tau),
\end{align*}
for all $s' \in \mathcal{S}$ at each step $h \in [H]$ of each trajectory $\tau \in [n]$.
\end{definition}
Definition~\ref{def:compliance} implies that the data collecting process should satisfy the Markov property. At each step $h \in [H]$ of each trajectory $\tau \in [n]$, $s_{h+1}^\tau$ depends on $\{(s_h^j, a_h^j)\}_{j=1}^\tau \cup \{s_{h+1}^j\}_{j=1}^{\tau - 1}$ only via $(s_h^\tau, a_h^\tau)$ and the transition dynamics $P_h$ of the underlying MDP $\mathcal{M}$. Thus the randomness in the $\{s_h^j, a_h^\tau, s_{h+1}^j\}_{j=1}^{\tau - 1}$ is completely captured by $(s_h^\tau, a_h^\tau)$ when we examine the randomness in $s_{h+1}^\tau$.
\begin{assumption}\label{ass:datacollection}(Data Collecting Process \citep{jin2021pessimism}) The offline source and target task datasets the learner has access to are compliant with their respective underlying MDPs. 
\end{assumption}
Assumption~\ref{ass:datacollection} is a weak assumption and captures several scenarios. (i) An experimenter collected the data according to a fixed policy, (ii) Experimenter sequentially improved the policy to collect data using any online RL algorithm, thus allowing the trajectories to be interdependent across each other, (iii) experimenter collected the data by taking actions arbitrarily, say randomly or even any adaptive or adversarial manner and doesn't need to conform to any fixed policy. The important part is that Assumption~\ref{ass:datacollection} doesn't require the dataset to well explore that state action space which is often the case with offline datasets such as electronic health records or human driving trajectories for autonomous driving.


\section{Representation Learning}
Recall from Definition~\ref{def:lowrankmdpdefinition}, the transition dynamics of low rank MDPs can be expressed as a function of the representation. In our setting, all the MDPs have a shared representation (Assumption~\ref{ass:commonrep}). Note that Assumption~\ref{ass:linearspan} implies that the transition dynamics of the target task lies in a linear span of the transition dynamics of the source tasks. Thus obtaining an estimate of the representation from the source tasks significantly reduces the sample complexity in the target task,  since it allows the learner to model the transition dynamics of the target task in terms of this learnt representation. In this section we discuss the challenges of obtaining a good representation estimate without any coverage assumptions on the offline datasets and describe our methodology to overcome these challenges.


\subsection{Learning a Joint Representation}
In order to learn a joint representation from the source tasks, for every $h \in [H]$ we perform a Maximum Likelihood Estimate (MLE) using the union of data across all source tasks as described below:
\begin{align}\label{eq:mle}
    &\small{\widehat{\mu}_{1:K;h}, \widehat{\phi}_h = \amax_{\mu_{1:K} \in \Upsilon, \phi \in \Phi} \sum_{i = 1}^K \sum_{\tau=1}^{N_S}  \log\mu^\top_i(s^{i;\tau}_{h+1}) \phi(s^{i;\tau}_h, a^{i;\tau}_h),}
\end{align}
where $\Upsilon$ and $\Phi$ are finite hypothesis classes and we are working in the realizable setting, i.e. $\mu^\ast_{1:K;h} \in \Upsilon, \phi_h^\ast \in \Phi$.
\begin{remark}\textbf{(Computation of MLE)} The MLE estimation is in general a non-convex optimization problem when $\phi$ and $\mu$ are general nonlinear function approximators. However, this is treated as a standard
supervised learning ERM oracle in the literature \citep{uehara2021representation, agarwal2020flambe, agarwal2023provable}.
For special cases where the MDP is tabular or linear, the MLE
objective is convex and the optimal solution has closed-form. 
\end{remark}
\subsection{Pointwise Uncertainty in Learnt Representation}
Since we do not assume any coverage conditions on the collected datasets, the representation learnt by Equation~\eqref{eq:mle} is likely to have estimation uncertainties. However, the magnitude of uncertainty for certain state-action pairs might be larger compared to others due to poor exploration. It is therefore desirable to quantify pointwise uncertainty in the estimation which is formally defined below.
\begin{definition} (Pointwise Uncertainty in Transition Dynamics) Given an arbitrary transition dynamics $\widehat{P}: \mathcal{S} \times \mathcal{A} \rightarrow \mathcal{S}$, its misspecification error at some state action pair $(s, a) \in \mathcal{S} \times \mathcal{A}$ w.r.t. the true transition dynamics $P^\ast$ is defined as $\Delta_{\widehat{P}}(s, a) = \|\widehat{P}(\cdot | s, a) - P^\ast(\cdot | s, a)\|^2_{TV}$. 
\end{definition}
In the context of low rank setting, the learner estimates the transition dynamics for task $i$ as $\widehat{P}_{i;h}(\cdot|s,a) = \widehat{\mu}_{i;h}^\top(\cdot)\widehat{\phi}_h(s, a)$, where $\widehat{\mu}_{i;h}, \widehat{\phi}_h$ are obtained from Equation~\eqref{eq:mle}. As discussed in \citep{uehara2021representation} the joint estimation of $\mu$ and $\phi$ in Equation~\eqref{eq:mle} is an instance of non-linear function approximation. Therefore one cannot get pointwise uncertainty quantification via the typically used linear-regression based analysis. Due to this bottleneck, prior works extensively study this problem in the online setting to ensure good exploration and uniform coverage \citep{agarwal2023provable} or in the offline scenario by imposing the strict assumption that all source datasets have uniformly explored all state action pairs.  This allows for the construction of a uniform confidence bound, i.e. $\epsilon = \min_{(s, a) \in \mathcal{S} \times \mathcal{A}} \Delta_{\widehat{P}}(s, a)$ before transferring this representation for planning in the target task. The magnitude of $\epsilon$ impacts the suboptimality of the learnt policy for the target task.  However, without uniform coverage assumptions this approach could be detrimental because even one failure mode, i.e. failing to explore some state action-pair even in one source task could lead to a large value of $\epsilon$, rendering the suboptimality of the target task policy meaningless. This motivates us to develop an algorithm to quantify pointwise uncertainty in the transition dynamics estimation.

First we state a guarantee on the estimates in Equation~\eqref{eq:mle}. The following lemma states that the sum of the pointwise errors in the transition dynamics averaged over the points in the source datasets is upper bounded with high probability.
\begin{restatable}{lemma}{averagemleguarantee}\label{lem:mle_guarantee}
Let $\{\widehat{\mu}_{i;h}\}_{i \in [K]}, \widehat{\phi}_h$ be the learned MLE estimates from Equation~\eqref{eq:mle}. Then with probability at least $1 - \delta$ we have the following bound:
\begin{align}\label{eqn:mle_guarantee}
\sum_{i=1}^{K}\underbrace{\sum_{\tau=1}^{N_S}\frac{\|\widehat{\mu}_{i;h}(\cdot)^\top \widehat{\phi_h}(s^{i;\tau}_h, a^{i;\tau}_h) -  \mu^\ast_{i;h}(\cdot)^\top \phi_h^\ast(s^{i;\tau}_h, a^{i;\tau}_h)\|^2_{TV}}{N_S}}_{\text{\rm{average error on source $i$'s dataset}}}  \leq \frac{2(\log(|\Phi|/\delta) + K\log(|\Upsilon|))}{N_S}. 
\end{align}
\end{restatable}
It would useful to use the average sense guarantee in Lemma~\ref{lem:mle_guarantee} to derive pointwise guarantees. To work our way towards this goal, we introduce the following concept.

\noindent\textbf{Neighborhood Density:} 
We borrow ideas from non-parametric estimation literature \citep{epanechnikov1969non, kaplan1958nonparametric} where the probability density at some point is estimated based on the observed data in its neighborhood (for, e.g., kernel density estimation \citep{chen2017tutorial}). Since \eqref{eq:mle} uses non-linear function estimation, we first need to formalize the concept of neighborhood in our setting. The $\nu_i$-neighborhood occupancy density at some $(s, a)$ in the dataset for source task $i$ denoted by $D_{i;h}^{\nu_i}(s, a)$ is the fraction of points in the dataset within a distance of $\nu_i$ of $(s, a)$ in the representation space $\mathbb{R}^d$ and is defined in Equation~\eqref{eq:neighborhooddensity}. $D_{i;h}^{\nu_i}(s, a)$ essentially \textit{quantifies how well a dataset explores regions around $(s, a)$ in the reprsentation space}. In the following lemma we focus our attention on quantifying the pointwise uncertainty for source task $i$, where the transition dynamics is estimated as $\widehat{P}_{i;h}(\cdot|s,a) = \widehat{\mu}_{i;h}(\cdot)^\top \widehat{\phi_h}(s, a)$. 

\begin{restatable}{lemma}{confidence}\label{lem:confidence}
Let $\Delta^i_{\widehat{\phi}_h, \widehat{\mu}_{i;h}}(s, a) = \|\widehat{\mu}_{i;h}(\cdot)^\top \widehat{\phi_h}(s, a) -  \mu^\ast_{i;h}(\cdot)^\top \phi_h^\ast(s, a)\|^2_{TV}$ denote the transition dynamics misspecification for source task $i$ at time $h$ for any $(s, a) \in \mathcal{S} \times \mathcal{A}$ specified by representations $\widehat{\phi}_h, \widehat{\mu}_{i;h}$ learnt from Equation~\eqref{eq:mle}. For some $\nu_i \geq 0$, $\Delta^i_{\widehat{\phi}_h, \widehat{\mu}_{i;h}}(s, a)$ can be upper bounded as follows:
    \begin{align*}
        \small{\Delta^i_{\widehat{\phi}_h, \widehat{\mu}_{i;h}}(s, a) \leq \underbrace{2 d \cdot  \nu_i}_{\rm bias} + }\small{\underbrace{\sum_{(s', a', \cdot) \in \mathcal{D}_{i; h}}\frac{\|\widehat{\mu}_{i;h}(\cdot)^\top \widehat{\phi_h}(s', a') -  \mu^\ast_{i;h}(\cdot)^\top \phi_h^\ast(s', a')\|^2_{TV}}{N_S D_{i;h}^{\nu_i}(s, a)}}_{\rm variance}.} 
    \end{align*}
\end{restatable}
Note that the variance term is the average error on source $i$'s dataset (Lemma~\ref{lem:mle_guarantee}) divided by the $\nu_i$-neighborhood occupancy density $D_{i;h}^{\nu_i}(s, a)$. 
Since, $D_{i;h}^{\nu_i}(s, a)$ is a non-decreasing function of $\nu_i$, the variance term is non-increasing in $\nu_i$, whereas the bias term is increasing in $\nu_i$.  Thus there is a bias-variance tradeoff in choosing $\nu_i$. We utilize this idea in Algorithm~\ref{alg:pointwiseerror}, which solves an optimization problem Equation~\eqref{eqn:optthreshold} to optimally balance out the total variance and bias across all source tasks to return the effective occupancy density $D_h(s, a)$, as defined in Equation~\eqref{eq:effectivedensity}. Now, we are ready to state our main result and provide a proof sketch to highlight the main ideas.
\begin{algorithm}[t] 
    \caption{Effective Occupancy Density}
    \begin{algorithmic}[1]
        \STATE \textbf{Input:} Source Datasets : $\mathcal{D}_{i;h} = \{(s_h^{i; \tau}, a_h^{i; \tau})\}_{\tau=1}^{N_S}$ for all $i \in [K], h \in [H]$, $C = \frac{\log(|\Phi|/\delta) + K\log{|\Upsilon|}}{N_S d}$.
        \FOR{$h \in [H]$}
        \FOR{$(s, a) \in \mathcal{S} \times \mathcal{A}$}
        \STATE Define $\nu_i$-neighborhood occupancy density 
        \begin{align}\label{eq:neighborhooddensity}
            & D_{i;h}^{\nu_i}(s, a) = \frac{1}{N_S}\inf_{\phi \in \Phi}
        \max_{\mathcal{C} \subseteq \mathcal{D}_{i;h}} |\mathcal{C}| \\
        &\text{such that } \|\phi(s, a) -  \phi(s', a')\|_1 \leq \nu_i, \; \forall (s', a') \in \mathcal{C} \nonumber.
        \end{align}
        \STATE Solve $\{\nu_1, \ldots, \nu_K\} \subseteq \mathbb{R}^K_{+}$ such that:
        \begin{align}\label{eqn:optthreshold}
             \min_{i \in [K]}  D_{i;h}^{\nu_i}(s, a) \cdot \sum_{i \in [K]} \nu_i = C.
        \end{align}
        \STATE Define effective occupancy density:
        \begin{align}\label{eq:effectivedensity}
            D_h(s, a) = \frac{C}{\sum_{i \in [K]} \nu_i}.
        \end{align}
        \ENDFOR
        \ENDFOR
    \end{algorithmic}
    \label{alg:pointwiseerror}
\end{algorithm}

\begin{restatable}{theorem}{representationbound}
\label{thm:representationbound} (Representation Transfer Error): Let $P_{h}^\ast(\cdot|s_h, a_h)$ denote the true transition dynamics of the target task, and $ \widehat{\phi_h}(s, a)$ be the learnt representation from Equation~\eqref{eq:mle}. For all $h \in [H], (s, a) \in \mathcal{S} \times \mathcal{A}$, there exists $\mu'_h : \mathcal{S} \rightarrow \mathbb{R}^d$ such that the following bound holds with probability at least $1 - \delta$  
    \begin{align*}
         \small{\|\mu'_h(\cdot)^\top \widehat{\phi_h}(s, a) -  P_{h}^\ast(\cdot| s, a)\|_{TV}} \small{\leq 2\alpha_{\mathrm{max}}\sqrt{\frac{K}{N_S}\frac{\log(|\Phi|/\delta) + K\log{|\Upsilon|}}{D_h(s,a)}},}
    \end{align*}
where $D_h(s,a)$ is the effective occupancy density as computed in Algorithm~\ref{alg:pointwiseerror}.
\end{restatable}
\noindent\textit{Proof Sketch:} We show in Lemma~\ref{lem:error_as_sum} that there exists a transition model linear in the learnt representation $\widehat{\phi}_h$ such that the model misspecification error for the target task can be upper bounded in terms of the model misspecification errors of the individual source tasks, i.e. $\sum_{i \in [K]} \Delta^i_{\widehat{\phi}_h, \widehat{\mu}_{i;h}}(s, a)$. The sum of the variance terms can be upper bounded with high probability by utilizing the MLE guarantee in Lemma~\ref{lem:mle_guarantee} with an additional multiplicative factor of the Importance sampling ratio $\frac{1}{\min_{i \in [K]}D_{i;h}^{\nu_i}(s, a)}$. The solution of the optimization problem in Equation~\eqref{eqn:optthreshold} in Algorithm~\ref{alg:pointwiseerror} optimally balances out the overall variance with the sum of bias terms. 

One of the main implications of Theorem~\ref{thm:representationbound} is that the learner doesn't need to impose the strict assumption that every source task has extensively explored every state action pair in order to have a uniformly low representation transfer error. In fact in the following corollary we present a much more relaxed yet sufficient condition to ensure uniformly low estimation error. If for some $\nu' \in (0,1]$ and every $(s, a) \in \mathcal{S} \times \mathcal{A}$ the following set of inequality constraints admits a feasible solution:
\begin{align}\label{eq:satisfiability}
    &\small{\left(\frac{1}{K}\sum_{i \in [K]} \frac{1}{D^{\nu_i}_{i;h}(s, a)}\right)^{-1} \geq  \frac{\log(|\Phi|/\delta) + K\log{|\Upsilon|}}{ N_S \cdot d \cdot \nu'}}\\
    &\small{\text{such that} \quad \frac{1}{K}\sum_{i \in [K]}\nu_i \leq \nu',} \nonumber
\end{align}
then the representation error is uniformly upper bounded and scales as $\sqrt{\nu'}$ as formalized below.

\begin{restatable}{corollary}{conditionforgoodreptransfer}\label{cor:condition_for_good_rep_transfer}
Let $P_{h}^\ast(\cdot|s_h, a_h)$ denote the true transition dynamics of the target task, and $ \widehat{\phi_h}(s, a)$ be the learnt representation from Equation~\eqref{eq:mle}. For all $h \in [H], (s, a) \in \mathcal{S} \times \mathcal{A}$, there exists $\mu'_h : \mathcal{S} \rightarrow \mathbb{R}^d$ such that the following bound holds with probability at least $1 - \delta$ 
\begin{align*}
        \|\mu'_h(\cdot)^\top \widehat{\phi_h}(s, a) -  P_{h}^\ast(\cdot| s, a)\|_{TV} \leq 2\alpha_{\mathrm{max}} K \sqrt{ d \cdot \nu'},
    \end{align*}
under the condition in Equation~\eqref{eq:satisfiability}.
\end{restatable}
 Thus we only need the harmonic means of the neighborhood densities to be lower bounded under an upper bounded average neighborhood size in order to get a uniformly low representation transfer error.
 
\section{Representation transfer in Target Task}

\begin{algorithm}
    \caption{Pessimisitc RepTransfer (PRT)}
    \begin{algorithmic}[1]
        \STATE \textbf{Input:} Target Dataset $\mathcal{D} = \{(s_h^\tau, a_h^\tau)\}_{\tau, h = 1}^{n, H}$, Learnt Representation $\widehat{\phi}_h(\cdot, \cdot)$, RepTransfer bound $\epsilon(\cdot, \cdot)$, $\beta, \lambda$.
        \STATE \textbf{Initialization:} Set $\widehat{V}_{H+1}(\cdot) \leftarrow 0$.
        \FOR{$h = H, H-1, \ldots, 1$}
        \STATE Set $\Lambda_h \leftarrow \frac{1}{n}\left(\sum_{\tau=1}^n \widehat{\phi}_h(s_h^\tau, a_h^\tau) \widehat{\phi}_h(s_h^\tau, a_h^\tau)^\top + \lambda \mathbb{I}\right)$.
        \STATE Set $\widehat{w}_h \leftarrow \Lambda_h^{-1}\left(\frac{1}{n}\sum_{\tau=1}^n \widehat{\phi}_h(s_h^\tau, a_h^\tau)\cdot \widehat{V}_{h+1}(s_{h+1}^\tau)\right)$.
        \STATE Set $\epsilon_h \leftarrow \sqrt{\frac{1}{n}\sum_{\tau=1}^n \epsilon(s_h^\tau, a_h^\tau)^2}$
        \STATE Set $\Gamma_h(\cdot, \cdot) \leftarrow H(\beta + \epsilon_h)\cdot\|\widehat{\phi}_h (\cdot, \cdot)\|_{\Lambda_h} +  H\epsilon(\cdot, \cdot)$.
        \STATE Set $\overline{Q}_h(\cdot, \cdot) \leftarrow r_h(\cdot, \cdot) + \widehat{\phi}_h(\cdot, \cdot)^\top \widehat{w}_h - \Gamma_h(\cdot, \cdot)$.
        \STATE Set $\widehat{Q}_h(\cdot, \cdot) \leftarrow \min\{\overline{Q}_h(\cdot, \cdot), H - h + 1\}^+$.
        \STATE Set $\widehat{\pi}_h(\cdot|\cdot) \leftarrow \arg\max_{\pi_h} \widehat{Q}_h(\cdot, \cdot)^\top \pi_h$.
        \STATE Set $\widehat{V}_h(\cdot) \leftarrow \mathbb{E}_{a \sim \widehat{\pi}_h(\cdot|\cdot)} \widehat{Q}_h(\cdot, a)$.
        \ENDFOR
        \STATE \textbf{Return} $\widehat{\pi} = \{\widehat{\pi}_h(\cdot|\cdot)\}_{h=1}^H$.
    \end{algorithmic}
    \label{alg:PEVIRT}
\end{algorithm}


In this section we present PessimisticRepTransfer (Algorithm~\ref{alg:PEVIRT}) for policy planning in the target task using the learnt representation and state our main results. We are now only focused on the target task, so subsequent notations are simplified. 

First we present a brief overview of the standard Value Iteration algorithm \citep{sutton2018reinforcement}, which under the assumption of known transition dynamics $P^\ast_h(\cdot|s,a)$ returns the optimal policy. Recall the definition of the $Q$-function : $Q_{h}(s, a) = r_h(s, a) + \mathbb{E}_{s' \sim P^\ast_h(\cdot|s,a)}V_{h+1}(s')$.  The Value Iteration Algorithm initializes $V_{H+1} = 0$ and goes backwards by setting the policy $\pi_H(s) = \amax_{a \in \mathcal{A}} Q_{H}(s, a)$, and the corresponding value for this policy  $V_{H}(s) = \max_{a \in \mathcal{A}} Q_{H}(s, a)$. Doing this iteratively for all $h=H-1,\ldots, 1$, the learner is able to obtain the optimal policy $\pi_1^\ast, \ldots, \pi_H^\ast$. 

However, since $P^\ast_h(\cdot|s,a)$ is unknown in our setting, the learner is unable to accurately compute $\mathbb{E}_{s' \sim P^\ast_h(\cdot|s,a)}V_{h+1}(s')$ at any arbitrary step $h$. However based on the available offline data and using the low-rank structure, the learner can form an estimate  $\mathbb{E}_{s' \sim \widehat{P}_h(\cdot|s,a)} \widehat{V}_{h+1}(s') = \widehat{\phi}_h(s, a)^\top \widehat{w}_h$, using Least Squares regression (see Lines 4-5 in Algorithm~\ref{alg:PEVIRT}). Since this estimate is likely to have uncertainties, before constructing the $Q$-function it is necessary to penalise every $(s, a)$ based on how uncertain the estimation is. The following lemma introduces such an uncertainty quantifier $\Gamma_h(s, a)$ with high probability.
\begin{restatable}{lemma}{uncertaintyquantifier}\label{lem:uncertaintyquantifier}
In Algorithm \ref{alg:PEVIRT}, setting $\lambda = 1, \beta = c \cdot d \sqrt{\zeta}$ and $\epsilon(s, a) = 2\alpha_{\mathrm{max}}\sqrt{K\frac{\log(2|\Phi|/\delta) + K\log{|\Upsilon|}}{N_S D_h(s, a)}}$ where $\zeta = \frac{\log (4dHn / \delta)}{n}$. Here $c \geq 1$ is an absolute constant and $\delta \in (0, 1)$ is the confidence parameter and $\epsilon_h =\sqrt{\frac{1}{n}\sum_{\tau=1}^n \epsilon(s_h^\tau, a_h^\tau)^2}$. Define the event $\mathcal{E} :$
    \begin{align*}
        \small{ \{\Gamma_h(s, a) = |\mathbb{E}_{s' \sim P^\ast_h(\cdot|s,a)} \widehat{V}_{h+1}(s') - \mathbb{E}_{s' \sim \widehat{P}_h(\cdot|s,a)} \widehat{V}_{h+1}(s')|} \leq H (\beta + \epsilon_h)\cdot \|\widehat{\phi}_h (s, a)\|_{\Lambda_h} + H \epsilon(s, a)\}.
    \end{align*}
Then $\mathcal{E}$ satisfies $P_{\mathcal{D}}(\mathcal{E}) \geq 1 - \delta$, where $P_{\mathcal{D}}$ is the data generating process for the tasks satisfying Assumption~\ref{ass:datacollection}.
\end{restatable}


By penalizing the $Q$-functions by the uncertainty quantifier (lines 8-9 Algorithm~\ref{alg:PEVIRT}), the learner chooses the policy as the action maximizing the $Q-$function for each corresponding state (line 10 Algorithm~\ref{alg:PEVIRT}). Doing it for all steps $h \in [H]$ as described in Algorithm~\ref{alg:PEVIRT} gives the target policy. We now state a result the quality of this policy below:
\begin{restatable}{theorem}{suboptimalitygap}\label{thm:suboptbound}
    Let $\widehat{\pi}$ be the output of Algorithm \ref{alg:PEVIRT}. Then with probability at least $1 - \delta$ 
    \begin{align*}
    \small{\textrm{SubOpt}(\widehat{\pi}, s) \leq 2 H \sum_{h=1}^H \mathbb{E}_{(s_h, a_h) \sim \pi^\ast, P_h^\ast}\bigg[ \underbrace{\epsilon(s_h, a_h)}_{\text{\rm{source coverage on $\pi^\ast$}}} +} \small{(\beta + \underbrace{\epsilon_h}_{\text{\rm{source coverage on target}}}) \cdot \underbrace{\|\widehat{\phi}_h (s_h, a_h)\|_{\Lambda_h}}_{\text{\rm{target coverage on $\pi^\ast$}}}|s_1 = s\bigg].}
    \end{align*}
Here the expectation is taken with respect to the optimal policy $\pi^\star$ of the true underlying MDP of the target task. $\epsilon(s, a) = 2\alpha_{\mathrm{max}}\sqrt{K\frac{\log(2|\Phi|/\delta) + K\log{|\Upsilon|}}{N_S D_h(s, a)}}$, $\beta=c\cdot d\sqrt{\zeta}$, where $\zeta = \frac{\log (4dHn / \delta)}{n}$ and $c \geq 1$ is an absolute constant and $\epsilon_h =\sqrt{\frac{1}{n}\sum_{\tau=1}^n \epsilon(s_h^\tau, a_h^\tau)^2}$.
\end{restatable}

Below we discuss the factors affecting the suboptimality of the learnt target policy:
\begin{enumerate}
[itemsep=-5pt, topsep=-5pt]
\item \textbf{Source Tasks' Coverage on Target Task's Optimal Policy $\pi^\ast$: }  The source tasks' should have sufficient samples along the trajectory of the optimal policy of the target task.
\begin{align*}
     \small{\sum_{h=1}^H \mathbb{E}_{(s_h, a_h) \sim \pi^\ast, P_h^\ast} \sqrt{\frac{1}{D_h(s_h, a_h)}}.}
\end{align*}
    \item 
\textbf{Source Tasks' Coverage on the offline samples from the Target Task : } Let $d_h(\cdot, \cdot)$ denote the target task's occupancy density based on the offline dataset $\mathcal{D}_h = \{s_h^\tau, a_h^\tau\}_{\tau=1}^n$. Evaluating the term $\epsilon_h$ we get:
\begin{align*}
    \small{\epsilon_h \propto \sqrt{\frac{1}{n}\sum_{\tau=1}^n \frac{1}{D_h(s_h^\tau, a_h^\tau)}} = \sqrt{\sum_{(s,a) \in \mathcal{S} \times \mathcal{A}}\frac{d_h(s, a)}{D_h(s, a)}}.}
\end{align*}
Note that $\epsilon_h$ doesn't depend on $\pi^\ast$ or $P_h^\ast$. In order for the representation transfer to be effective, this term implies that the source tasks' must have sufficient coverage at all points covered in the target task. 

\item \textbf{Target Task's Coverage on its Optimal Policy $\pi^\ast$ : } $\Lambda_h = \frac{1}{n}\left(\sum_{\tau=1}^n \widehat{\phi}_h(s_h^\tau, a_h^\tau) \widehat{\phi}_h(s_h^\tau, a_h^\tau)^\top + \lambda \mathbb{I}\right)$ indicates the empirical covariance of the samples from the target task. For any arbitrary $(s_h, a_h)$, the term $\|\widehat{\phi}_h (s_h, a_h)\|_{\Lambda_h} = \widehat{\phi}_h (s_h, a_h)^\top \Lambda_h^{-1} \widehat{\phi}_h (s_h, a_h)$ indicates how well $(s_h, a_h)$ is covered by the offline samples from the target dataset. The suboptimality gap depends on how well the offline samples from the target task covers the trajectory of the taget task's optimal policy, i.e.,
\begin{align*}
    \small{2H \sum_{h=1}^H \mathbb{E}_{(s_h, a_h) \sim \pi^\ast, P_h^\ast} (\beta + \epsilon_h) \cdot \|\widehat{\phi}_h (s_h, a_h)\|_{\Lambda_h}.}
\end{align*}
\end{enumerate}

\subsection{Well Explored Source and Task Datasets}
We wish to study the suboptimality rates as a function of the number of source and target task samples. We examine this under the assumption that the data collecting process work with well exploratory policies, formally defined below.
\begin{assumption}\label{ass:exploratorypolicy} (Bounded Density in Representation Space) Let $\pi_i = \{\pi_{i,1}, \ldots, \pi_{i,H}\}$ denote the policy that collects offline data for source task i. A feature map $\phi \in \Phi$ defines a distribution $d^{\pi_i, \phi}_{i, h}(\cdot)$ in the representation space $\mathbb{R}^d$. We assume that there exists policy $\bar{\pi}_i$ such that we can lower bound the density in the representation space, i.e.
    \begin{align*}
        \inf_{\phi \in \Phi} \inf_{x \in \mathbb{R}^d} d^{\bar{\pi}_i, \phi}_{i, h}(x) \geq \psi \quad \text{and} \quad \sup_{\phi \in \Phi} \sup_{x \in \mathbb{R}^d} d^{\bar{\pi}_i, \phi}_{i, h}(x) \leq 1,
    \end{align*}
for all $h \in [H]$.
\end{assumption}
Note that by Definition~\ref{def:lowrankmdpdefinition}, every feature map $\phi \in \Phi$, satisfies $\|\phi(s, a)\|_2 \leq 1\; \forall (s, a) \in \mathcal{S} \times \mathcal{A}$. Thus the representation space in $\mathbb{R}^d$ is the unit $\ell_2$ norm ball in $d$ dimensions, $\mathcal{B}_2^d$ which is a compact set. Assumption~\ref{ass:exploratorypolicy} thus states existence of policy with bounded density only on a compact set instead of the raw state action space which can be infinite.
\begin{assumption}\label{ass:targetwellexplored}\citep{jin2021pessimism}
There exists a policy $\bar{\pi} = \{\bar{\pi}_1, \ldots, \bar{\pi}_H\}$ for the target task such that 
\begin{align*}
    \inf_{\phi \in \Phi}\lambda_{\mathrm{\min}}(\Sigma_h^\phi) \geq c/d \quad \text{where} \quad \Sigma_h^\phi = \mathbb{E}_{(s_h, a_h) \sim \bar{\pi}, P^\ast}[\phi_h(s, a)\phi_h(s, a)^\top].
\end{align*}
\end{assumption}
The following corollary gives a high probability bound on the target policy suboptimality as a function of the number of source and target task samples.
\begin{restatable}{corollary}{samplecomplexity}\label{cor:sample_complexity}
Let $\bar{\pi}$ be a policy satisfying Assumption~\ref{ass:targetwellexplored} and $\{\bar{\pi}_1, \ldots, \bar{\pi}_K\}$ be policies satisfying Assumption~\ref{ass:exploratorypolicy}. Suppose $n$ i.i.d. trajectories are sampled from the target task by policy $\bar{\pi}$ and $N_S$ i.i.d. trajectories are sampled from each task i by policy $\pi_i$.  Then with probability at least $1 - \delta$, the suboptimality gap is upper bounded as:
\begin{align*}
    &\textrm{SubOpt}(\widehat{\pi}, s) \leq \Tilde{\mathcal{O}}({N_S}^{-\frac{1}{4d}}n^{-\frac{1}{2}}H^2d^{\frac{3}{2}}K^{\frac{3}{4}}\sqrt{\log(1/\delta)}).
\end{align*}
\end{restatable}

\section{Experiments}
\begin{table*}[t]
    \centering
    \begin{tabular}{|c|c|ccc|c|}
    \hline
    Source Trajectories ($N_S$) &    Target Trajectories($n$) &  RT-LSVI & RT-LSVI-LCB & PRT (Ours) &  RT-LSVI-UCB\\
    \hline
    \multirow{3}{*}{500} & 150 &0.396 & 0.736 & \textbf{1.0} & \multirow{2}{*}{0.049}\\
        \cline{2-5}
         & 200 &0.416 & 0.812 & \textbf{1.0} & \multirow{2}{*}{(n=50000)}\\
        \cline{2-5}
        & 250 &0.504 & 0.892 & \textbf{1.0} & \\
        \hline
        \multirow{3}{*}{1000} & 150 &0.724 & 0.760 & \textbf{1.0} & \multirow{2}{*}{0.072}\\
        \cline{2-5}
        & 200 &0.864 & 0.880 & \textbf{1.0} & \multirow{2}{*}{(n=50000)}\\
        \cline{2-5}
        & 250 &0.940 & 0.960 & \textbf{1.0} & \\
        \hline
        \multirow{3}{*}{1500} & 150 &0.764 & 0.768 & \textbf{1.0} & \multirow{2}{*}{1.0}\\
        \cline{2-5}
        & 200 &0.880 & 0.892 & \textbf{1.0} & \multirow{2}{*}{(n=572)}\\
        \cline{2-5}
        & 250 &0.964 & 0.984 & \textbf{1.0} &\\
        \hline
    \end{tabular}
    \caption{Average Rewards for CombLock across different algorithms and varying number of samples. We observe that when $N_S$ is large, all algorithms perform well (Row 3). However, when $N_S$ is small, our algorithm significantly outperforms the baselines (Rows 1-2). It is essential to highlight that the online algorithm excels when the learned representation is accurate, as seen in Row 3.   However, its performance deteriorates and fails to converge when the learned representation is inaccurate, as observed in Rows 1 and 2. This underscores the sensitivity of the baselines to the precision of the learned representation.}
    \label{tab:cumulative_rewards}
\end{table*}
\begin{figure}
    \centering
    \includegraphics[width=0.48\textwidth]{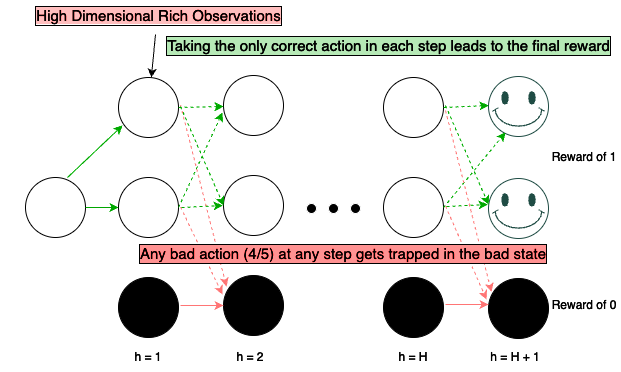}
    \caption{ A visualization of the rich observation comblock environment. Our experiment uses $K=5$ source tasks, $H=5$ time steps and 5 actions in each step. See Appendix~\ref{appendix:experiments} for details.}
    \label{fig:comblock}
\end{figure}
In this section we empirically study\footnote{All our code is available at \url{https://anonymous.4open.science/r/PessimisticRepTransfer-DBDE}} the benefits of penalizing the learnt representation in offline Multi-Task Transfer RL. We ask the following questions:
\begin{enumerate}
    \item Does uncertainty quantification in the learnt representation reduce sample complexity of both source and task datasets?
    \item Does running online algorithms such as UCB with inaccurate representation lead to convergence to suboptimal target policies?
\end{enumerate}
Our experiments suggest affirmative answer to both the questions above. We use the high dimesnional rich observation Combination Lock (comblock) benchmark (see Figure~\ref{fig:comblock}).

\noindent\textbf{Offline Dataset Construction: } We use the Exploratory Policy Search (EPS) Algorithm proposed by \citep{agarwal2023provable} to identify exploratory policies for the source and target task. \textit{Note that exploratory policies aim to cover as much of the feature space and are potentially very different from the optimal policy}. Next we independently and identically sampled $N_S$ trajectories from each of the source tasks and $n$ trajectories from the target task to construct our offline datasets.

\noindent\textbf{Baselines: } All the baselines considered in our study leverage the representation learned from source tasks' offline datasets, obtained through Maximum Likelihood Estimation as described in Equation~\eqref{eq:mle}. The algorithm employed for the target task varies across these baselines. \textbf{RT-LSVI} uses the LSVI (Least Squares Value Iteration) algorithm \citep{sutton2018reinforcement}, \textbf{RT-LSVI-LCB} uses the Lower Confidence Bound (LCB) algorithm \citep{jin2021pessimism}, \textbf{PRT} denotes our Algorithm~\ref{alg:PEVIRT}. These 3 are purely offline algorithms designed to work with the target task's offline dataset. \textbf{RT-LSVI-UCB} uses the learnt representation like the other baselines, but then can adaptively collect samples from the target task using the Upper Confidence Bound (UCB) algorithm \citep{sutton2018reinforcement}. In Table~\ref{tab:cumulative_rewards}, we vary the number of source trajectories $N_S$ and target trajectories $n$, reporting the average reward (over 50 runs) for all baselines. For \textbf{RT-LSVI-UCB}, $n$ is the number of trajectories for the algorithm to converge and is reported in parenthesis (we terminate when $n=50000$ if it fails to converge). 


\noindent\textbf{Few-Shot Learning:} In scenarios where the source tasks benefit from well explored datasets (i.e., large $N_S$), the representation transfer error is uniformly low. Baseline models demonstrate strong performance under these well-covered conditions, as evident in Row 3 of Table~\ref{tab:cumulative_rewards}. Our focus, however, lies in situations where the source datasets are less explored, indicating a small $N_S$ (see Rows 1-2 of Table~\ref{tab:cumulative_rewards}). In such cases, our representation transfer penalty becomes crucial for selectively penalizing estimated representations for specific state-action pairs. We observe that both \textbf{RT-LSVI} and \textbf{RT-LSVI-LCB}, assuming the learned representation as ground truth, struggle to reach the optimal solution in these less-explored settings. While \textbf{RT-LSVI-LCB} performs better than \textbf{RT-LSVI} by penalizing points in the target dataset, it still falls short. On the other hand, \textbf{RT-LSVI-UCB} fails to recover the optimal policy even after 50000 episodes due to reliance on an inaccurate representation. Notably, our proposed algorithm, \textbf{PRT}, stands out as the only method capable of optimally solving the target task. 
These results empirically complement our theoretical analysis of uncertainty quantification in representation transfer.

\section{Conclusion}
We address the challenges of offline representation transfer in low-rank MDPs which share the same representation space under a very relaxed assumption on the offline datasets: the trajectories should be compliant to the underlying MDPs. A notable contribution is the algorithmic construction of pointwise uncertainty quantifiers for the learned representation. We demonstrate via theoretical analysis and a numerical experiment that incorporating uncertainty in the learned representation yields an effective policy for the target task. Below we highlight several directions of future work.
\\
\\
\noindent\textbf{Future Work.} Working completely in the offline setting means the learner incurs an irreducible suboptimality from the error in the learnt representation. However, if the learner had online access to only the target task, then theoretical analysis of actively reducing the representation error is an interesting direction of future work. This has recently been experimentally studied and found to be beneficial in the context of language models \citep{bhatt2024experimental} where the sequential decision making task is next word prediction and typically pre-trained language models are fine tuned to achieve few-shot generalization to target tasks. 

Another promising direction of future work is the problem of source task selection. Typically domain experts are needed to select source tasks relevant for the corresponding target task. However, with the availability of offline datasets from a large number of source tasks available online necessitates principled approaches to select a small subset of tasks that are relevant to the target task. \citep{chen2022active} study this in the context of multi-task learning with linear prediction models and it would be interesting to extend it to multi-task RL. 
\clearpage
\bibliography{bib}
\newpage
\appendix
\section{Missing Algorithms}\label{appendix:algorithms}
In this section we present the algorithms we used for our baselines. All of these work under the assumption of a known representation. For our experiments, the estimated representation $\widehat{\phi}_h(\cdot, \cdot)$ is used by all the algorithms. 
\begin{algorithm}
    \caption{LSVI with Known Representation}
    \begin{algorithmic}
        \STATE \textbf{Input:} Target Dataset $\mathcal{D} = \{(s_h^\tau, a_h^\tau)\}_{\tau, h = 1}^{n, H}$, Known Representation $\phi_h(\cdot, \cdot)$.
        \STATE \textbf{Initialization:} Set $\widehat{V}_{H+1}(\cdot) \leftarrow 0$.
        \FOR{$h = H, H-1, \ldots, 1$}
        \STATE Set $\Lambda_h \leftarrow \frac{1}{n}\left(\sum_{\tau=1}^n \phi_h(s_h^\tau, a_h^\tau) \phi_h(s_h^\tau, a_h^\tau)^\top + \lambda \mathbb{I}\right)$.
        \STATE Set $\widehat{w}_h \leftarrow \Lambda_h^{-1}\left(\frac{1}{n}\sum_{\tau=1}^n \phi_h(s_h^\tau, a_h^\tau)\cdot \widehat{V}_{h+1}(s_{h+1}^\tau)\right)$.
        \STATE Set $\overline{Q}_h(\cdot, \cdot) \leftarrow r_h(\cdot, \cdot) + \phi_h(\cdot, \cdot)^\top \widehat{w}_h$.
        \STATE Set $\widehat{Q}_h(\cdot, \cdot) \leftarrow \min\{\overline{Q}_h(\cdot, \cdot), H - h + 1\}^+$.
        \STATE Set $\widehat{\pi}_h(\cdot|\cdot) \leftarrow \arg\max_{\pi_h} \widehat{Q}_h(\cdot, \cdot)^\top \pi_h$.
        \STATE Set $\widehat{V}_h(\cdot) \leftarrow \mathbb{E}_{a \sim \widehat{\pi}_h(\cdot|\cdot)} \widehat{Q}_h(\cdot, a)$.
        \ENDFOR
        \STATE \textbf{Return} $\widehat{\pi} = \{\widehat{\pi}_h(\cdot|\cdot)\}_{h=1}^H$.
    \end{algorithmic}
    \label{alg:LSVI}
\end{algorithm}
\begin{algorithm}
    \caption{LSVI-LCB with Known Representation}
    \begin{algorithmic}
        \STATE \textbf{Input:} Target Dataset $\mathcal{D} = \{(s_h^\tau, a_h^\tau)\}_{\tau, h = 1}^{n, H}$, Known Representation $\phi_h(\cdot, \cdot)$.
        \STATE \textbf{Initialization:} Set $\widehat{V}_{H+1}(\cdot) \leftarrow 0$.
        \FOR{$h = H, H-1, \ldots, 1$}
        \STATE Set $\Lambda_h \leftarrow \frac{1}{n}\left(\sum_{\tau=1}^n \phi_h(s_h^\tau, a_h^\tau) \phi_h(s_h^\tau, a_h^\tau)^\top + \lambda \mathbb{I}\right)$.
        \STATE Set $\widehat{w}_h \leftarrow \Lambda_h^{-1}\left(\frac{1}{n}\sum_{\tau=1}^n \phi_h(s_h^\tau, a_h^\tau)\cdot \widehat{V}_{h+1}(s_{h+1}^\tau)\right)$.
        \STATE Set $\Gamma_h(\cdot, \cdot) \leftarrow \beta  \cdot\|\phi_h (\cdot, \cdot)\|_{\Lambda_h}$.
        \STATE Set $\overline{Q}_h(\cdot, \cdot) \leftarrow r_h(\cdot, \cdot) + \phi_h(\cdot, \cdot)^\top \widehat{w}_h - \Gamma_h(\cdot, \cdot)$.
        \STATE Set $\widehat{Q}_h(\cdot, \cdot) \leftarrow \min\{\overline{Q}_h(\cdot, \cdot), H - h + 1\}^+$.
        \STATE Set $\widehat{\pi}_h(\cdot|\cdot) \leftarrow \arg\max_{\pi_h} \widehat{Q}_h(\cdot, \cdot)^\top \pi_h$.
        \ENDFOR
        \STATE \textbf{Return} $\widehat{\pi} = \{\widehat{\pi}_h(\cdot|\cdot)\}_{h=1}^H$.
        \STATE Set $\widehat{V}_h(\cdot) \leftarrow \mathbb{E}_{a \sim \widehat{\pi}_h(\cdot|\cdot)} \widehat{Q}_h(\cdot, a)$.
    \end{algorithmic}
    \label{alg:LCB}
\end{algorithm}
\begin{algorithm}
    \caption{LSVI-UCB with Known Representation}
    \begin{algorithmic}
        \STATE \textbf{Input:} Known Representation $\phi_h(\cdot, \cdot)$.
        \STATE \textbf{Initialization:} Set $\widehat{V}_{H+1}(\cdot) \leftarrow 0$, Randomly initialize $\{\widehat{\pi}_h\}_{h \in [H]}$.
        \FOR{$l=1, \ldots, n$}
        \STATE Sample $s_1^l \sim d_1$.
        \FOR{$h=1, \ldots, H$}
        \STATE Perform $a_h^l \sim \widehat{\pi}_h(\cdot|s_h^l)$.
        \STATE Collect $s_{h+1}^{l} \sim P_h(\cdot|s_h^l, a_h^l)$.
        \ENDFOR
        \FOR{$h = H, H-1, \ldots, 1$}
        \STATE Set $\Lambda_h \leftarrow \frac{1}{l}\left(\sum_{\tau=1}^l \phi_h(s_h^\tau, a_h^\tau) \phi_h(s_h^\tau, a_h^\tau)^\top + \lambda \mathbb{I}\right)$.
        \STATE Set $\widehat{w}_h \leftarrow \Lambda_h^{-1}\left(\frac{1}{l}\sum_{\tau=1}^l \phi_h(s_h^\tau, a_h^\tau)\cdot \widehat{V}_{h+1}(s_{h+1}^\tau)\right)$.
        \STATE Set $\Gamma_h(\cdot, \cdot) \leftarrow \beta  \cdot\|\phi_h (\cdot, \cdot)\|_{\Lambda_h}$.
        \STATE Set $\overline{Q}_h(\cdot, \cdot) \leftarrow r_h(\cdot, \cdot) + \phi_h(\cdot, \cdot)^\top \widehat{w}_h + \Gamma_h(\cdot, \cdot)$.
        \STATE Set $\widehat{Q}_h(\cdot, \cdot) \leftarrow \min\{\overline{Q}_h(\cdot, \cdot), H - h + 1\}^+$.
        \STATE Set $\widehat{\pi}_h(\cdot|\cdot) \leftarrow \arg\max_{\pi_h} \widehat{Q}_h(\cdot, \cdot)^\top \pi_h$.
        \STATE Set $\widehat{V}_h(\cdot) \leftarrow \mathbb{E}_{a \sim \widehat{\pi}_h(\cdot|\cdot)} \widehat{Q}_h(\cdot, a)$.
        \ENDFOR
        
        \ENDFOR
        \STATE \textbf{Return} $\widehat{\pi} = \{\widehat{\pi}_h(\cdot|\cdot)\}_{h=1}^H$.
    \end{algorithmic}
    \label{alg:UCB}
\end{algorithm}
\section{Proof of MLE Guarantee}\label{appendix:mle_proofs}
We first state an auxiliary lemma which allows us to work our way to the MLE guarantee for Equation~\ref{eq:mle}.
\begin{lemma}\label{lem:mle} 
Consider a class of conditional probability distribution functions $\mathcal{F} : \{f | f(y| x) \rightarrow [0, 1]\}$. Suppose we have data samples $\mathcal{D} = \{(x_i, y_i)\}_{i=1}^n \subset \mathcal{X} \times \mathcal{Y}$ where $y|x \sim f^\ast (\cdot | x)$ ($f^\ast \in \mathcal{F}$). We find an MLE estimate:
\begin{align*}
    \widehat{f} = \amax_{f \in \mathcal{F}} \sum_{i=1}^n \log f(y_i | x_i). 
\end{align*}
Then the following bound holds with probability at least $1 - \delta$:
    \begin{align*}
        \frac{1}{n}\sum_{i \in [n]} \|\widehat{f}(\cdot | x_i) - f^\ast( \cdot | x_i)\|^2_{TV} \leq \frac{2\log(|\mathcal{F}|/\delta)}{n}.
    \end{align*}
\end{lemma}
\begin{proof}
    Given an set of points $\{x_1, \ldots, x_n\}$, we observe samples $\{y_1, \ldots, y_n\}$ from $f^\ast$. We wish to understand how well the estimate $\widehat{f}$ captures the randomness in the $\mathcal{Y}$ space on the empirical distribution over $\{x_1, \ldots, x_n\}$. $ \frac{1}{n}\sum_{i \in [n]} \|\widehat{f}(\cdot | x_i) - f^\ast( \cdot | x_i)\|^2_{TV}$ is a measure of the quality of this estimate. 
    
    We invoke Theorem 18 from \citep{agarwal2020flambe}, with a slight variation. Given offline source data $\mathcal{D} = \{(x_i, y_i)\}_{i \in  [n]}$, we create a tangent sequence $\mathcal{D}' = \{(x_i', y_i')\}_{i \in [n]}$, where $x_i' = x_i$ and $y_i' \sim f^\ast(\cdot | x_i')$. Rest of the proof follows after making this choice of $\mathcal{D}'$. In \citep{agarwal2020flambe}, they consider the randomness in the $\mathcal{X}$ space as well, but since we are working with a offline dataset, we don't need to take that into account.
\end{proof}

\averagemleguarantee*
\begin{proof}
    This follows from Lemma~\ref{lem:mle}, where the function class is expressed as $\mathcal{F} = \Phi \times \Upsilon^K$ and the number of samples is $N_S$.
\end{proof}

\section{Proof of Theorem~\ref{thm:representationbound}}\label{appendix:theorem1_proof}
\subsection{Error Bounds for one Source Task}
We first derive an upper bound on the pointwise uncertainty error for any low-rank MDP in the following lemma.
\begin{lemma}\label{lem:lipschitzerror}
    For all $\phi \in \Phi, \mu \in \Upsilon$, the pointwise model misspecification $\Delta_{\phi,\mu}(\cdot,\cdot)$ can be bounded as
    \begin{align*}
        \Delta_{\phi, \mu}(s, a) \leq \frac{1}{|\mathcal{D}|} \left(\sum_{(s',a') \in \mathcal{D}} \Delta_{\phi, \mu}(s', a') + 2d \sup_{\phi \in \Phi} \sum_{(s',a') \in \mathcal{D}}  \|\phi(s, a) - \phi(s', a')\|_1\right),
    \end{align*}
for all $(s,a) \in \mathcal{S} \times \mathcal{A}$ and all $\mathcal{D} \subseteq \mathcal{S} \times \mathcal{A}$.
\end{lemma}
\begin{proof}
    We use $\|\mu(\cdot)\| = \|\int_{s \in \mathcal{S}} \mathbf{d}\mu(s)\|$. By Definition~\ref{def:lowrankmdpdefinition}, choosing $g(s) = 1 \; \forall s \in \mathcal{S}$, we have $\|\mu(\cdot)\|_2 \leq \sqrt{d}$. By Cauchy Schwartz inequality $\|\mu(\cdot)\|_1 \leq d$. Noting that that the total variation distance between two distributions is the $\ell_1$ norm of their difference, we can write:
\begin{align*}
    &\bigg|\Delta_{\phi, \mu}(s', a') - \Delta_{\phi, \mu}(s, a)\bigg|\\
    &= \bigg|\frac{1}{4} \|\mu^\top(\cdot)\phi(s',a') - {\mu^{\ast}}^\top(\cdot)\phi^\ast(s',a')\|_1^2 - \frac{1}{4} \|\mu^\top(\cdot)\phi(s,a) - {\mu^{\ast}}^\top(\cdot)\phi^\ast(s,a)\|_1^2\bigg|\\
    &= \frac{1}{4} \bigg|\|\mu^\top(\cdot)\phi(s',a') - {\mu^{\ast}}^\top(\cdot)\phi^\ast(s',a')\|_1 - \|\mu^\top(\cdot)\phi(s,a) - {\mu^{\ast}}^\top(\cdot)\phi^\ast(s,a)\|_1\bigg|\\&\cdot(\|\mu^\top(\cdot)\phi(s',a') - {\mu^{\ast}}^\top(\cdot)\phi^\ast(s',a')\|_1 + \|\mu^\top(\cdot)\phi(s,a) - {\mu^{\ast}}^\top(\cdot)\phi^\ast(s,a)\|_1)\\
    &\leq \bigg|\|\mu^\top(\cdot)\phi(s',a') - {\mu^{\ast}}^\top(\cdot)\phi^\ast(s',a')\|_1 - \|\mu^\top(\cdot)\phi(s,a) - {\mu^{\ast}}^\top(\cdot)\phi^\ast(s,a)\|_1\bigg| \\&\text{(Since $\|\mu^\top(\cdot)\phi(s,a)\|_1 = 1$ $\forall (s,a) \in \mathcal{S} \times \mathcal{A}$ and $\forall \phi \in \Phi, \mu \in \Upsilon$; it is a probability distribution)}\\
    &\leq  \|\mu^\top(\cdot)\phi(s',a') - \mu^\top(\cdot)\phi(s,a)\|_1 + \|{\mu^{\ast}}^\top(\cdot)\phi^\ast(s',a') - {\mu^{\ast}}^\top(\cdot)\phi^\ast(s,a)\|_1 \quad \text{(Triangle Inequality)}\\
    &\leq d\left(\|\phi(s',a') - \phi(s,a)\|_1 + \|\phi^\ast(s',a') - \phi^\ast(s,a)\|_1\right) \quad \text{(Since $\|\mu(\cdot)\|_1 \leq d \quad\forall \mu \in \Upsilon$)}.\\
\end{align*}
Now given a set of state-action pairs $\mathcal{D} \subseteq \mathcal{S} \times \mathcal{A}$, we can write: 
\begin{align*}
    &\bigg|\frac{1}{|\mathcal{D}|}
    \sum_{(s',a') \in \mathcal{D}} (\Delta_{\phi, \mu}(s, a) - \Delta_{\phi, \mu}(s', a'))\bigg|\\ 
    &\leq \frac{1}{|\mathcal{D}|}\sum_{(s',a') \in \mathcal{D}} \bigg|\left(\Delta_{\phi, \mu}(s, a) - \Delta_{\phi, \mu}(s', a')\right)\bigg|\\
    &\leq \frac{1}{|\mathcal{D}|}\sum_{(s',a') \in \mathcal{D}} d \left(\|\phi(s',a') - \phi(s,a)\|_1 + \|\phi^\ast(s',a') - \phi^\ast(s,a)\|_1 \right)\\
    &\leq 2d \sup_{\phi \in \Phi} \sum_{(s',a') \in \mathcal{D}}  \|\phi(s, a) - \phi(s', a')\|_1.
\end{align*}
This completes the proof.
\end{proof}
Note that the lemma above allows us to write the uncertainty at some point $(s, a)$ in terms of the distances in the representation space for any arbitrary $\mathcal{D} \subseteq \mathcal{S} \times \mathcal{A}$. In the following lemma we are going to restrict $\mathcal{D}$ to be a subset of the offline dataset and use the MLE guarantee (Lemma~\ref{lem:mle_guarantee}).

\confidence*
\begin{proof}
For a given $(s, a)$, choose $\mathcal{D} = S_{i; h}(s, a, \nu_i)$ where 
\begin{align*}
            & S_{i;h}(s, a, \nu_i) = \frac{1}{N_S}\inf_{\phi \in \Phi}
        \amax_{\mathcal{C} \subseteq \mathcal{D}_{i;h}} |\mathcal{C}| \\
        &\text{such that } \|\phi(s, a) -  \phi(s', a')\|_1 \leq \nu_i, \; \forall (s', a') \in \mathcal{C},
        \end{align*}
as the subset of datapoints optimizing Equation~\ref{eq:neighborhooddensity} in Algorithm~\ref{alg:pointwiseerror}. Plugging this choice of $\mathcal{D}$ in Lemma~\ref{lem:lipschitzerror}, we can write:
\begin{align*}
    \Delta_{\phi, \mu}(s, a) \leq \frac{1}{|S_{i; h}(s, a, \nu_i)|} \left(\sum_{(s',a') \in S_{i; h}(s, a, \nu_i)} \Delta_{\phi, \mu}(s', a') + 2d \sup_{\phi \in \Phi} \sum_{(s',a') \in S_{i; h}(s, a, \nu_i)}  \|\phi(s, a) - \phi(s', a')\|_1\right).
\end{align*}
The second term on the right hand side is $\leq \nu_i$ by the condition of the optimization problem. 
Now we will use importance sampling (IS) to bound the first term on the right hand side by the average error on dataset (Lemma~\ref{lem:mle_guarantee}).
Consider a support as the collection of state action pairs in $\mathcal{D}_{i;h}$. For the expression above, the probability density is:
\begin{equation*}
q_i(s',a')=
\begin{cases}
     \frac{1}{|S_{i; h}(s, a, \nu_i)|} & \text{if} (s', a') \in S_{i; h}(s, a, \nu_i) \\
     0 & \text{otherwise}
    \end{cases}
\end{equation*}
The probability distribution for the average error on dataset is uniform $\frac{1}{N_S}$ on the support.
Therefore, the IS ratio $\max_{(s', a', \cdot) \in \mathcal{D}_{i;h}} \frac{q_i(s',a')}{p_i(s',a')} =  \frac{N_S}{|S_{i; h}(s, a, \nu_i)|} = \frac{1}{D_{i;h}^{\nu_i}(s, a)}$. Hence, we can write:
\begin{align*}
   \Delta^i_{\widehat{\phi}, \widehat{\mu}_{i;h}}(s, a) \leq \frac{1}{D_{i;h}^{\nu_i}(s, a)} \sum_{i \in [K]} \frac{1}{N_S}\sum_{(s', a') \in \mathcal{D}_{i; h}}\|\widehat{\mu}_{i;h}(\cdot)^\top \widehat{\phi_h}(s', a') -  \mu^\ast_{i;h}(\cdot)^\top \phi_h^\ast(s', a')\|^2_{TV} + 2 d \cdot \nu_i.
\end{align*}    
This completes the proof.
\end{proof}
\subsection{Error Bounds for Target Task}
First we show the existence of a transition function linear in $\widehat{\phi}_h$, such that the pointwise error of this transition function with respect to the true transition dynamics of the target task can be decomposed into the sum of pointwise errors of the individual source tasks.
\begin{lemma}\label{lem:error_as_sum}
Let $P_{h}^\ast(\cdot|s_h, a_h)$ denote the true transition dynamics of the target task, and $ \widehat{\phi_h}(s, a)$ be the learnt representation from Equation~\eqref{eq:mle}. For all $h \in [H], (s, a) \in \mathcal{S} \times \mathcal{A}$, there exists $\mu'_h : \mathcal{S} \rightarrow \mathbb{R}^d$ such that:
    \begin{align*}
        \|\mu'_h(\cdot)^\top \widehat{\phi_h}(s, a) -  P_{h}^\ast(\cdot| s, a)\|^2_{TV} \leq \alpha^2_{\mathrm{max}} K \sum_{i \in [K]} \Delta^i_{\widehat{\phi}_h, \widehat{\mu}_{i;h}}(s, a).
    \end{align*}
\end{lemma}
\begin{proof}
Denote $\mu'_h(s') = \sum_{i \in [K]} \alpha_{i;h}(s')\widehat{\mu}_{i:h}(s')$ where $\alpha_{i;h}(s')$ is as defined in Assumption~\ref{ass:linearspan}.
    \begin{align*}
        \Delta_{\widehat{\phi}_h, \mu'_h}(s,a) &=\|\mu'_h(\cdot)^\top \widehat{\phi_h}(s, a) -  P_{h}^\ast(\cdot| s, a)\|^2_{TV}\\
        &= \|\mu'_h(\cdot)^\top \widehat{\phi_h}(s, a) -  \mu_h^\ast(\cdot)^\top \phi_h^\ast(s, a)\|^2_{TV}\\
        &= \|\sum_{i \in [K]} \alpha_{i;h}(\cdot)\left(\widehat{\mu}_{i;h}(\cdot)^\top \widehat{\phi_h}(s, a) -  \mu^\ast_{i;h}(\cdot)^\top \phi_h^\ast(s, a)\right)\|^2_{TV} \quad \text{(By Assumption~\ref{ass:linearspan})}\\
        &\leq \alpha^2_{\mathrm{max}} \|\sum_{i \in [K]} \widehat{\mu}_{i;h}(\cdot)^\top \widehat{\phi_h}(s, a) -  \mu^\ast_{i;k}(\cdot)^\top \phi_h^\ast(s, a)\|^2_{TV}\\
        &\leq \alpha^2_{\mathrm{max}} K \sum_{i \in [K]}\|\widehat{\mu}_{i;h}(\cdot)^\top \widehat{\phi_h}(s, a) -  \mu^\ast_{i;h}(\cdot)^\top \phi_h^\ast(s, a)\|^2_{TV} \quad \text{(By Cauchy Schwartz)}\\
        &= \alpha^2_{\mathrm{max}} K \sum_{i \in [K]} \Delta^i_{\widehat{\phi}_h, \widehat{\mu}_{i;h}}(s, a).
    \end{align*}    
\end{proof}
Now we show that the solution of Algorithm~\ref{alg:pointwiseerror}, allows to use the MLE guarantee (Lemma~\ref{lem:mle_guarantee}) to get a high probability pointwise uncertainty error bound for the target task. 
\representationbound*
\begin{proof}
    
Let's use $\Delta_{\widehat{\phi}_h, \mu'}(s, a)$ to denote $\|\mu'_h(\cdot)^\top \widehat{\phi_h}(s, a) -  P_{h}^\ast(\cdot| s, a)\|_{TV}^2$. By Lemma~\ref{lem:error_as_sum}, we can write:
\begin{align*}
         \Delta_{\widehat{\phi}_h, \mu'}(s, a) \leq \alpha^2_{\mathrm{max}} K \sum_{i \in [K]} \Delta^i_{\widehat{\phi}_h, \widehat{\mu}_{i;h}}(s, a).
    \end{align*}
For some choice of $\{\nu_1, \ldots, \nu_K\}$, using Lemma~\ref{lem:confidence} for the right hand side, we can write:


\begin{align*}
   \Delta_{\widehat{\phi}, \mu'}(s,a) \leq \alpha^2_{\mathrm{max}} K \left( \sum_{i \in [K]} \frac{1}{D_{i;h}^{\nu_i}(s, a)} \frac{1}{N_S}\sum_{(s', a') \in \mathcal{D}_{i; h}}\|\widehat{\mu}_{i;h}(\cdot)^\top \widehat{\phi_h}(s', a') -  \mu^\ast_{i;h}(\cdot)^\top \phi_h^\ast(s', a')\|^2_{TV} + 2 d L \sum_{i \in K} \nu_i\right).
\end{align*}
Invoking Lemma~\ref{lem:mle_guarantee}, with probability at least $1 - \delta$, we have:
\begin{align*}
     \Delta_{\widehat{\phi}, \mu'}(s,a) \leq \alpha^2_{\mathrm{max}} K \left(\max_{i \in [K]} \frac{2}{D_{i;h}^{\nu_i}(s, a)}\frac{\log(|\Phi|/\delta) + K\log{|\Upsilon|}}{N_S}  + 2 d L \sum_{i \in K} \nu_i \right). 
\end{align*}
Choosing $\{\nu_1, \ldots, \nu_K\}$ by Algorithm~\ref{alg:pointwiseerror} we can write:
\begin{align*}
    \Delta_{\widehat{\phi}, \mu'}(s,a) \leq 4 \alpha^2_{\mathrm{max}} K \frac{\log(|\Phi|/\delta) + K\log{|\Upsilon|}}{N_S D_{h}(s, a)}.
\end{align*}
This completes the proof.
\end{proof}
\section{Proof of Corollary~\ref{cor:condition_for_good_rep_transfer}}
\conditionforgoodreptransfer*
\begin{proof}
For a given $(s, a) \in \mathcal{S} \times \mathcal{A}$, let $\nu_1, \ldots, \nu_K$ be such that the following are satisfied:
    \begin{align*}
    &\left(\frac{1}{K}\sum_{i \in [K]} \frac{1}{D^{\nu_i}_{i;h}(s, a)}\right)^{-1} \geq  \frac{\log(|\Phi|/\delta) + K\log{|\Upsilon|}}{ N_S \cdot d \cdot \nu'} \quad \text{and} \quad \frac{1}{K}\sum_{i \in [K]}\nu_i \leq \nu'.
\end{align*}
Given any arbitrary set of positive numbers $\{a_1, \ldots, a_K\}$, using the properties that $\mathbf{HM}(a_1, \ldots, a_K) \leq K \min\{a_1, \ldots, a_K\}$, we get:
$\min_{i \in [K]} D^{\nu_i}_{i;h}(s, a) \geq \frac{\log(|\Phi|/\delta) + K\log{|\Upsilon|}}{K\cdot N_S \cdot d \cdot \nu'}$. 

Since $\min_{i \in [K]} D^{\nu_i}_{i;h}(s, a) \cdot \sum_{i \in [K]} \nu_i$ is an increasing function in $\{\nu_1, \ldots, \nu_K\}$, thus there exists $\{\nu'_1, \ldots, \nu'_K\}$ such that $\nu_i' \geq \nu_i \; \forall i \in [K]$ and $\sum_{i \in [K]} \nu'_i = K \nu'$ satisfies : $\min_{i \in [K]} D^{\nu'_i}_{i;h}(s, a) \cdot \sum_{i \in [K]} \nu'_i \geq \frac{\log(|\Phi|/\delta) + K\log{|\Upsilon|}}{N_S \cdot d }$. 

Therefore there exists $\nu^\ast_1, \ldots, \nu^\ast_K$ such that $\nu^\ast_i \leq \nu'_i \; \forall i \in [K]$ and $\sum_{i \in [K]} \nu^\ast_i \leq K \nu'$ which is the solution of Equation~\eqref{eqn:optthreshold}. 
Therefore by Equation~\eqref{eq:effectivedensity}:
$D_h(s, a) = \frac{\log(|\Phi|/\delta) + K\log{|\Upsilon|}}{N_S \cdot d \cdot \sum_{i\in[K]}\nu^\ast_i} \geq \frac{\log(|\Phi|/\delta) + K\log{|\Upsilon|}}{K\cdot N_S \cdot d \cdot \nu'}$. Plugging this back in Theorem~\ref{thm:representationbound} gives us the desired result.
\end{proof}
\section{Proof of Theorem~\ref{thm:suboptbound}}\label{appendix:proof_thm2}
We introduce the following standard definition to ease the presentation of the results in this section.
\begin{definition}\label{def:transitiondefinition}
    (Transition Operator): Given any function $f : \mathcal{S} \rightarrow \mathbb{R}$, the Transition operator $P_h$ at step $h \in [H]$ is defined as:
    \begin{align*}
        (P_h f)(s, a) = \mathbb{E}_{s' \sim P_h(\cdot|s, a)} f(s').
    \end{align*}
\end{definition}
The following lemma states that for a low-rank MDP the Transition Operator $P_h$ can be written as a linear function of the representation $\phi_h(\cdot,\cdot)$.
\begin{lemma}
    For a low rank MDP, given any function $f : \mathcal{S} \rightarrow \mathbb{R}$ there exists an unknown $w_h \in \mathbb{R}^d$ such that
    \begin{align*}
        (P_h f)(s, a) = \phi_h(s, a)^\top w_h.
    \end{align*}
\end{lemma}
\begin{proof}
    By Definition \ref{def:lowrankmdpdefinition} and \ref{def:transitiondefinition} we have:
    \begin{align*}
        (P_h f)(s, a) &= \int_{\mathcal{S}} \phi_h(s, a)^\top \mu_h(s')f(s')\mathbf{d}s'\\
        &= \phi_h(s, a)^\top \int_{\mathcal{S}} \mu_h(s')f(s')\mathbf{d}s'.
    \end{align*}
Thus $w_h = \int_{\mathcal{S}} \mu_h(s')f(s')\mathbf{d}s'$.
\end{proof}
Since the true representation $\phi_h^\ast(\cdot, \cdot)$ is not known, a key step is proving under Theorem~\ref{thm:representationbound}, the existence of a transition operator $P'_h$ with high probability which is linear in the learnt representation $\widehat{\phi}_h(\cdot,\cdot)$ that is close to the true transition operator $P_h^\ast$.
\begin{lemma}\label{lem:intermediatetransition} Let $\widehat{\phi}_h(\cdot, \cdot)$ be a representation from Equation~\eqref{eq:mle}. Given any function $f : \mathcal{S} \rightarrow \mathbb{R}$, there exists an unknown $w'_h \in \mathbb{R}^d$, such that $(P'_h f)(s,a) = \widehat{\phi}_h(s, a)^\top w'_h$ satisfies the following bound with probability at least $1 - \delta/2$:
    \begin{align*}
    |(P_h^\ast f)(s, a) - (P'_h f)(s, a)| \leq \max_{\mathcal{S}}|f| \cdot \epsilon(s, a),
    \end{align*}
where $\epsilon(s, a) = 2\alpha_{\mathrm{max}}\sqrt{K\frac{\log(2|\Phi|/\delta) + K\log{|\Upsilon|}}{n_{\nu; h}(s, a)}}$.
\end{lemma}
\begin{proof}
 Let $w'_h = \int_{\mathcal{S}}f(s')\mu'_h(s')\mathbf{d}s'$, where $\mu'_h(\cdot)$ is as defined in Theorem~\ref{thm:representationbound}. Then we have:
\begin{align*}
    &|(P_h^\ast f)(s, a) - (P'_h f)(s, a)|\\
    =&  |\int_{\mathcal{S}} P_h^\ast(s'|s,a)f(s')\mathbf{d}s' - \int_{\mathcal{S}} \widehat{\phi}_h(s,a)^\top \mu'_h(s')f(s')\mathbf{d}s'|\\
    \leq& \max_{\mathcal{S}}|f| \cdot|\int_{\mathcal{S}} P_h^\ast(s'|s,a)\mathbf{d}s' - \int_{\mathcal{S}} \widehat{\phi}_h(s,a)^\top \mu'_h(s')\mathbf{d}s'|.
\end{align*}
Now use Theorem~\ref{thm:representationbound} with tolerance $\delta/2$ to bound the second term with probability at least $1 - \delta/2$ completes the proof.
\end{proof}

\uncertaintyquantifier*
\begin{proof}
Note that $\Gamma_h(s, a) = |(P_h^\ast \widehat{V}_{h+1})(s, a) - (\widehat{P}_h \widehat{V}_{h+1})(s, a)|$. We now use triangle inequality to upper bound $\Gamma_h(s, a)$.
    \begin{align*}
        &|(P_h^\ast \widehat{V}_{h+1})(s, a) - (\widehat{P}_h \widehat{V}_{h+1})(s, a)|\\
        \leq & \underbrace{|(P_h^\ast \widehat{V}_{h+1})(s, a) - (P'_h \widehat{V}_{h+1})(s, a)|}_{(i)} + \underbrace{|(P'_h \widehat{V}_{h+1})(s, a) - (\widehat{P}_h \widehat{V}_{h+1})(s, a)|}_{(ii)}
    \end{align*}
$(i)$ can be bounded by Lemma \ref{lem:intermediatetransition}, with $f = \widehat{V}_{h+1}$ and $\max_{\mathcal{S}}|\widehat{V}_{h+1}| \leq H$. Thus $(i) \leq H\epsilon(s, a)$ with probability at least $1 - \delta/2$.

Let us define 
\begin{align*}
    w'_h = \int_{\mathcal{S}} \widehat{V}_{h+1}(s') \mu'_h(s')\mathbf{d}s'.
\end{align*}
Thus we can write 
\begin{align*}
    (P'_h \widehat{V}_{h+1})(s, a) = \widehat{\phi}_h(s, a)^\top w'_h.
\end{align*}

Now we analyse $(ii)$.
\begin{align*}
    (ii) &= \widehat{\phi}_h(s, a)^\top (w'_h - \widehat{w}_h)\\
    &= \widehat{\phi}_h(s, a)^\top \left(w'_h - \Lambda_h^{-1}(\frac{1}{n}\sum_{\tau=1}^n \widehat{\phi}_h(s_h^\tau, a_h^\tau)\cdot \widehat{V}_{h+1}(s_{h+1}^\tau))\right)\\
    &= \underbrace{\widehat{\phi}_h(s, a)^\top \left(w'_h - \Lambda_h^{-1} \left(\frac{1}{n}\sum_{\tau=1}^n \widehat{\phi}(s_h^\tau, a_h^\tau) \cdot (P'_h \widehat{V}_{h+1})(s_h^\tau, a_h^\tau)\right)\right)}_{(iii)} \\& - \underbrace{\widehat{\phi}_h(s, a)^\top \Lambda_h^{-1} \left(\frac{1}{n}\sum_{\tau=1}^n \widehat{\phi}(s_h^\tau, a_h^\tau) \cdot (\widehat{V}_{h+1}(s_{h+1}^\tau) - (P_h \widehat{V}_{h+1})(s_h^\tau, a_h^\tau)\right)}_{(iv)} \\ &+
    \underbrace{\widehat{\phi}_h(s, a)^\top \Lambda_h^{-1}\left(\frac{1}{n}\sum_{\tau=1}^n \widehat{\phi}(s_h^\tau, a_h^\tau)\cdot(P'_h \widehat{V}_{h+1})(s_h^\tau, a_h^\tau) - (P_h \widehat{V}_{h+1})(s_h^\tau, a_h^\tau))\right)}_{(v)}.
\end{align*}
Using triangle inequality we get $|(ii)| \leq |(iii)| + |(iv)| + |(v)|$.
The analysis for $|(iii)|, |(iv)|$ follows similarly to the proof of Lemma 5.2 in \cite{jin2021pessimism}. We state the bounds:
\begin{align*}
    |(iii)| \leq H\sqrt{\frac{d\lambda}{n}} \sqrt{\left(\widehat{\phi}_h(s, a)^\top \Lambda_h^{-1} \widehat{\phi}_h(s, a)\right)} \leq H \frac{\beta}{2} \sqrt{\left(\widehat{\phi}_h(s, a)^\top \Lambda_h^{-1} \widehat{\phi}_h(s, a)\right)}.
\end{align*}
\begin{align*}
    P_{\mathcal{D}}\left(|(iv)| \leq H\frac{\beta}{2}\sqrt{\left(\widehat{\phi}_h(s, a)^\top \Lambda_h^{-1} \widehat{\phi}_h(s, a)\right)} \right) \geq 1 - \delta / 2,
\end{align*}
where $P_{\mathcal{D}}$ is the data generating distribution.
 
Let us now bound $|(v)|$.
\begin{align*}
    |(v)| &= |\widehat{\phi}_h(s, a)^\top \Lambda_h^{-1}\left(\frac{1}{n}\sum_{\tau=1}^n \widehat{\phi}(s_h^\tau, a_h^\tau)\cdot(P'_h \widehat{V}_{h+1})(s_h^\tau, a_h^\tau) - (P_h \widehat{V}_{h+1})(s_h^\tau, a_h^\tau))\right)|\\
    &\leq  \sqrt{\frac{1}{n}\left(\sum_{\tau=1}^n \left(\widehat{\phi}_h(s, a)^\top \Lambda_h^{-1} \widehat{\phi}(s_h^\tau, a_h^\tau)\right)^2\right)\cdot \left(\frac{1}{n}\sum_{\tau=1}^n\left((P'_h \widehat{V}_{h+1})(s_h^\tau, a_h^\tau) - (P_h \widehat{V}_{h+1})(s_h^\tau, a_h^\tau)\right)^2\right)}\\
\end{align*}
This follows from Cauchy Schwartz inequality. From Theorem~\ref{thm:representationbound} we can bound the second term with probability at least $1 - \delta/2$, where $\epsilon(s, a) = 2\alpha_{\mathrm{max}}\sqrt{K\frac{\log(2|\Phi|/\delta) + K\log{|\Upsilon|}}{n_{\nu; h}(s, a)}}$ and we get:
\begin{align*}
    |(v)| &\leq  H\cdot \sqrt{\left(\widehat{\phi}_h(s, a)^\top \Lambda_h^{-1} \left(\frac{1}{n}\sum_{\tau=1}^n \widehat{\phi}(s_h^\tau, a_h^\tau)\widehat{\phi}(s_h^\tau, a_h^\tau)^\top \right) \Lambda_h^{-1} \widehat{\phi}_h(s, a)\right)} \sqrt{\frac{1}{n}\sum_{\tau=1}^n \epsilon(s_h^\tau, a_h^\tau)^2}\\
    &= H \cdot\sqrt{\left(\widehat{\phi}_h(s, a)^\top \Lambda_h^{-1} \left(\Lambda_h - \lambda \mathbb{I} \right) \Lambda_h^{-1} \widehat{\phi}_h(s, a)\right)}\sqrt{\frac{1}{n}\sum_{\tau=1}^n \epsilon(s_h^\tau, a_h^\tau)^2}\\
    &\leq H\cdot\sqrt{\left(\widehat{\phi}_h(s, a)^\top \Lambda_h^{-1} \widehat{\phi}_h(s, a)\right)}\sqrt{\frac{1}{n}\sum_{\tau=1}^n \epsilon(s_h^\tau, a_h^\tau)^2}
\end{align*}
The first inequality follows from Lemma \ref{lem:intermediatetransition} by noting $\max_{\mathcal{S}}\widehat{V}_{h+1} \leq H$. The second equation follows from definition of $\Lambda_h$.

Combining the bounds by taking union bound concludes the proof of the Lemma.
\end{proof}
\suboptimalitygap*
\begin{proof}
    Follows from Theorem 4.2 in \cite{jin2021pessimism} by plugging in uncertainty quantifiers $\Gamma_h(\cdot, \cdot)$ satisfying guarantees in Lemma \ref{lem:uncertaintyquantifier}.
\end{proof}
\section{Proof for Uniform Cover in Source Tasks (Corolloary~\ref{cor:sample_complexity})}\label{appendix:proof_cor2}
\subsection{Recap on Covering and Packing Numbers}
\begin{definition}($\nu$-Covering)
    Let $(\mathcal{V}, \|\cdot\|)$ be a normed space and $\mathcal{X} \subseteq \mathcal{V}$. A set $\mathcal{A}$ is called a $\nu$-covering of $\mathcal{X}$,  if for all $x \in \mathcal{X} \; \exists x' \in \mathcal{A}$ such that $\|x - x'\| \leq \nu$. The collection of such sets is denoted by denoted by $\mathcal{N}(\mathcal{X}, \|\cdot\|, \nu)$.
\end{definition}
\begin{definition}($\nu$-Covering Number) The size of the minimal set which is $\nu$-covering is defined as the $\nu$-covering number, that is
\begin{align*}
    N(\mathcal{X}, \|\cdot\|, \nu) = \min_{\mathcal{A} \in \mathcal{N}(\mathcal{X}, \|\cdot\|, \nu)} |A|.
\end{align*}
\end{definition}

\begin{definition}($\nu$-Packing)
    Let $(\mathcal{V}, \|\cdot\|)$ be a normed space and $\mathcal{X} \subseteq \mathcal{V}$. A set $\mathcal{A}$ is called a $\nu$-packing of $\mathcal{X}$,  if for all $x,x' \in \mathcal{A} \; \|x - x'\| \geq \nu$. The collection of such sets is denoted by denoted by $\mathcal{M}(\mathcal{X}, \|\cdot\|, \nu)$.
\end{definition}
\begin{definition}($\nu$-Packing Number) The size of the maximal set which is $\nu$-packing is defined as the $\nu$-packing number, that is
\begin{align*}
    M(\mathcal{X}, \|\cdot\|, \nu) = \max_{\mathcal{A} \in \mathcal{M}(\mathcal{X}, \|\cdot\|, \nu)} |A|.
\end{align*}
We introduce some notation and state some bounds on covering and packing numbers. The unit $\ell_p$ norm ball in $\mathbb{R}^d$ is defined as :
\begin{align*}
    \mathcal{B}_p^d = \{x | x \in \mathbb{R}^d\; ; \|x\|_p \leq 1\}.
\end{align*}
\end{definition}
The following lemmas are borrowed from \citep{zhou2002covering}.
\begin{lemma}\label{lem:normballcoveringbound}
The $\nu$-covering number of $\mathcal{B}_2^d$ satisfies:
    \begin{align*}
        N(\mathcal{B}_2^d, \|\cdot\|_1, \nu) \leq \bigg(\frac{3\sqrt{d}}{\nu}\bigg)^d.
    \end{align*}
\end{lemma}
\begin{lemma}
The $\nu$-packing number of $\mathcal{B}_1^d$ satisfies:
    \begin{align*}
        M(\mathcal{B}_1^d, \|\cdot\|_1, \nu) \geq \bigg(\frac{1}{\nu}\bigg)^d
    \end{align*}
\end{lemma}

\subsection{Some Results using Covering and Packing Numbers}
\begin{lemma}\label{lem:conditions_on_D_cover}
    A set $\mathcal{D} \subseteq \mathcal{X}$ is a $\nu$-covering of $\mathcal{X}$, i.e. $\mathcal{D} \in \mathcal{N}(\mathcal{X}, \|\cdot\|, \nu)$ if it is a $\nu/2$-covering for some $\mathcal{A} \in \mathcal{N}(\mathcal{X}, \|\cdot\|, {\nu/2})$. 
\end{lemma}
\begin{proof}
    Let us consider a set  $\mathcal{D} \subseteq \mathcal{X}$ that is a $\nu/2$-covering for some $\mathcal{A} \in \mathcal{N}(\mathcal{X}, \|\cdot\|, {\nu/2})$.
    Therefore for all $x' \in \mathcal{A}$ there exists $x'' \in \mathcal{D}$ such that $\|x'' - x'\| \leq \nu/2$. 

    Now by definition of $\mathcal{A}$, for every $x \in \mathcal{X}$ there exists $x' \in \mathcal{A}$ such that $\|x' - x\| \leq \nu/2$.

    For any $x \in \mathcal{X}$, by the existence results above there exists $x' \in \mathcal{A}, x'' \in \mathcal{D}$ such that $\|x'' - x'\| \leq \nu/2$ and $\|x' - x\| \leq \nu/2$. 

    Using triangle inequality:
    \begin{align*}
        \|x - x''\| &= \|(x - x') + (x' - x'')\| \\
        &\leq \|(x - x')\| + \|(x' - x'')\| \\
        &\leq \nu.
    \end{align*}
This proves that $\mathcal{D}$ is a $\nu$-covering of $\mathcal{X}$.
\end{proof}
\begin{lemma}\label{lem:boosting}
    If $\mathcal{D} \subseteq \mathcal{X}$ is a $\nu/ l$-covering of $\mathcal{X}$ then for every $x \in \mathcal{X}$ there exists at least $M(\mathcal{B}_1, \|\cdot\|_1, \frac{1}{l})$ number of points $x' \in \mathcal{D}$
    such that $\|x' - x\|_1 \leq \nu$. 
\end{lemma}
\begin{proof}
    Pick any $x \in \mathcal{X}$ and construct an $\ell_1$ norm ball of radius $\nu$ centered at $x$, $x + \nu\mathcal{B}_1$ . The maximum number  $\frac{\nu}{l}\mathcal{B}_1$ balls we can pack in  $\nu\mathcal{B}_1$ is given by $M(\mathcal{B}_1, \|\cdot\|_1, \frac{1}{l})$. Since, $\mathcal{D}$ covers $\mathcal{X}$ and consequently $x + \nu\mathcal{B}_1$, there are at least $M(\mathcal{B}_1, \|\cdot\|_1, \frac{1}{l})$ points in $\mathcal{D}$ that are contained in  $x + \nu\mathcal{B}_1$. 
\end{proof}
\section{Proof of Corollary~\ref{cor:sample_complexity}}
\begin{lemma}\label{lem:conditions_for_nu_cover}
    Let $\pi_i$ be a policy satisfying Assumption~\ref{ass:exploratorypolicy} used to collect $n$ i.i.d. trajectories. Let $\mathcal{D}^h_n = \{(s_1, a_1), \ldots, (s_n, a_n)\}$ denote the $n$ state action pairs in the offline dataset at time step $h$. Then with probability at least $1 - \delta$, for every $(s, a) \in \mathcal{S} \times \mathcal{A}$ and for all $h \in [H]$ there exists $(s', a') \in \mathcal{D}^h_n$ such that $\sup_{\phi \in \Phi} \|\phi(s, a) - \phi(s', a')\|_1 \leq \nu$ if 
    \begin{align*}
    n \geq C\nu^{-d},
\end{align*}
where $C = c^{-d}\psi^{-d}(6\sqrt{d})^d \cdot \left( d\log (6\sqrt{d} / \delta)\right)$.
\end{lemma}
\begin{proof}
The condition for every $(s, a) \in \mathcal{S} \times \mathcal{A}$ and for all $h \in [H]$ there exists $(s', a') \in \mathcal{D}^h_n$ such that $\sup_{\phi \in \Phi} \|\phi(s, a) - \phi(s', a')\|_1 \leq \nu$ implies that we need the offline dataset to be $\nu$-covering in the representation space. 
For a particular $\phi \in \Phi$, we use $\mathcal{D}_n^\phi = \{\phi(s, a) | (s, a) \in \mathcal{D}_n^h\}$ to denote the mappings of the state action pairs in $\mathcal{D}_n^h$ in the representation space. 

By Lemma~\ref{lem:conditions_on_D_cover} it is sufficient for $\mathcal{D}_n^\phi$ to be an $\nu/2$-covering of some $\mathcal{A} \in \mathcal{N}(\mathcal{B}_2^d,\|\cdot\|_1,{\nu/2})$ to be an $\nu$-covering of $\mathcal{B}_2^d$ (since the representation space is $\mathcal{B}_2^d$). We choose the minimal set $\mathcal{A}$ such that $|\mathcal{A}| = N(\mathcal{B}_2^d,\|\cdot\|_1,{\nu/2})$. We need to show that the worst case (over $\phi \in \Phi$) $\mathcal{D}_n^\phi$ is $\nu$-covering with high probability. 

Lets construct bins $\mathcal{A}_y = \{y' | y' \in \mathcal{B}_2^d, \|y' - y\|_1 \leq \nu/2\}$ for all $y \in \mathcal{A}$. Note that $\cup_{y \in \mathcal{A}} \mathcal{A}_y = \mathcal{B}_2^d$. These sets are $\ell_1$ norm balls of radius $\nu/2$, i.e. $\nu/2 \cdot \mathcal{B}_1^d$ if they lie in the complete interior of $\mathcal{B}_2^d$. For those sets on the boundary, their volume is at least some fraction $c$ times the volume of $\nu/2 \cdot \mathcal{B}_1^d$ since their center is within $\mathcal{B}_2^d$, for some finite $c < 1$. Thus we can argue
\begin{align*}
    &N(\mathcal{B}_2^d, \|\cdot\|_1, \nu/2) \mathrm{Vol}(\mathcal{A}_y) \geq c\mathrm{Vol}(\mathcal{B}_2^d) \quad \forall y \in \mathcal{A}\\
    &\implies \frac{\psi \cdot \mathrm{Vol}(\mathcal{A}_y)}{1 \cdot \mathrm{Vol}(\mathcal{B}_2^d)} \geq \frac{c\psi}{N(\mathcal{B}_2^d, \|\cdot\|_1, \nu/2)}.
\end{align*}
By Assumption~\ref{ass:exploratorypolicy}, we have $\inf_{\phi \in \Phi} \inf_{x \in \mathbb{R}^d} d^{\pi_i, \phi}_{i, h}(x) \geq \psi$ and $\sup_{\phi \in \Phi} \sup_{x \in \mathbb{R}^d} d^{\pi_i, \phi}_{i, h}(x) \geq 1$. Thus we can write 
\begin{align*}
    \inf_{\phi \in \Phi} P(x \in \mathcal{A}_y | x \sim d^{\pi_i, \phi}_{i, h}) \geq \frac{\psi \cdot \mathrm{Vol}(\mathcal{A}_y)}{1 \cdot \mathrm{Vol}(\mathcal{B}_2^d)} \geq \frac{c\psi}{N(\mathcal{B}_2^d, \|\cdot\|_1, \nu/2)} = p.
\end{align*}
This statement implies that the map of a randomly sampled state action pair via $\pi_i$ in the representation space lies in the bin $\mathcal{A}_y$ with probability at least $p$. 

We upper bound the probability that none of the $n$ i.i.d. draws lies in $\mathcal{A}_y$ as follows:
\begin{align*}
        \sup_{\phi \in \Phi} P(\not\exists x \in \mathcal{D}_n^\phi \; ; x \in \mathcal{A}_y) \leq (1 - p)^n \leq \exp{(-np)}.
\end{align*}

Now we wish to lower bound the probability that given $n$ i.i.d. draws we sample at least 1 point from each of these bins. This is achieved as follows:
\begin{align*}
   \inf_{\phi \in \Phi}  P(\cap_{y \in \mathcal{A}}\exists x \in \mathcal{D}_n \; ; x \in \mathcal{A}_y) &= 1 - \sup_{\phi \in \Phi} P(\cup_{y \in \mathcal{A}} \not\exists x \in \mathcal{D}_n \; ; x \in \mathcal{A}_y)\\ 
    &\geq 1 - \sum_{y \in \mathcal{A}} \sup_{\phi \in \Phi} P(\not\exists x \in \mathcal{D}_n \; ; x \in \mathcal{A}_y)\\
    &\geq 1 - N(\mathcal{B}_2^d, \|\cdot\|_1, \nu/2) \exp{(-np)}.
\end{align*}
For this probability to be greater than $1 - \delta$, we need
\begin{align*}
    n = \frac{N(\mathcal{B}_2^d, \|\cdot\|_1, \nu/2)}{c\psi}\log(\frac{N(\mathcal{B}_2^d, \|\cdot\|_1, \nu/2)}{\delta})
\end{align*} 
samples. Plugging in Lemma~\ref{lem:normballcoveringbound} 
\begin{align*}
    n = c^{-d}\psi^{-d}(6\sqrt{d})^d \cdot \left( d\log (6\sqrt{d} / \delta)\right) \nu^{-d}.
\end{align*}
samples are needed for this event to happen with probability at least $1 - \delta$. 
\end{proof}

\begin{lemma}
Let $\pi_i$ be a policy satisfying Assumption~\ref{ass:exploratorypolicy} used to collect $n$ i.i.d. trajectories. Let $\mathcal{D}^h_n = \{(s_1, a_1), \ldots, (s_n, a_n)\}$ denote the $n$ state action pairs in the offline dataset at time step $h$. Then with probability at least $1 - \delta$, for every $(s, a) \in \mathcal{S} \times \mathcal{A}$ and for all $h \in [H]$ there exists $\gamma \in \mathbb{N}$ state action pairs $(s', a') \in \mathcal{D}^h_n$ such that $\sup_{\phi \in \Phi} \|\phi(s, a) - \phi(s', a')\|_1 \leq \nu$ if 
    \begin{align*}
    n \geq C_1\nu^{-d},
\end{align*}
where $C_1 = \gamma^{d}c^{-d}\psi^{-d}(6\sqrt{d})^d \cdot \left( d\log (6\sqrt{d} / \delta)\right)$.
\end{lemma}
\begin{proof}
Setting $\nu \rightarrow \nu / \gamma$ in Lemma~\ref{lem:conditions_for_nu_cover} and using Lemma~\ref{lem:boosting} gives us the result.  
\end{proof}
\begin{restatable}{lemma}{uniformerrorbound}\label{lem:uniform_source_error}
    Let $\{\pi_1, \ldots, \pi_K\}$ be policies satisfying Assumption~\ref{ass:exploratorypolicy}.
    Suppose  $N_S$ i.i.d. trajectories are sampled from each task i by policy $\pi_i$, then with probability at least $1 - \delta$, for all $(s, a) \in \mathcal{S} \times \mathcal{A}$ we can upper bound the transition model estimation error as:
    \begin{align*}
        \|\mu'_h(\cdot)^\top \widehat{\phi_h}(s, a) -  P_{h}^\ast(\cdot| s, a)\|_{TV} \leq C_2 {N_S}^{-1/4d}\alpha_{\mathrm{\max}}\bigg((\log(2|\Phi|/\delta) + K\log{|\Upsilon|})\bigg)^{1/4}K^{3/4}d^{1/2}\psi^{-1/4},
    \end{align*}
    where $C_2$ is a finite constant.
\end{restatable}
\begin{proof}
Let us use $\epsilon$ to denote the desired error tolerance in the transition model. We choose $\gamma = 4\epsilon^{-2}\alpha_{\mathrm{max}}^2(K\log(2|\Phi|/\delta) + K^2\log{|\Upsilon|})$. If we sample  $N_S = \gamma^{d}c^{-d}\psi^{-d}(6\sqrt{d})^d \cdot \left( d\log (6\sqrt{d} / \delta)\right) \nu^{-d}$,
\begin{align*}
         &\min_{i \in [K]} D_{i;h}^{\nu}(s, a) N_S \geq \gamma = 4\epsilon^{-2}\alpha_{\mathrm{max}}^2(K\log(2|\Phi|/\delta) + K^2\log{|\Upsilon|})\\
        &\implies 2\alpha_{\mathrm{max}}\sqrt{\frac{K\log(2|\Phi|/\delta) + K^2\log{|\Upsilon|}}{\min_{i \in [K]} D_{i;h}^{\nu}(s, a) N_S}} \leq \epsilon\\
    \end{align*}


Writing the total variance as a function of $\nu$ (setting all $\nu_i$ identical to $\nu$). With probability at least $1 - \delta/2$ we can upper bound the total variance as
\begin{align*}
    \max_{i \in [K]} \frac{2}{D_{i;h}^{\nu}(s, a)}\frac{\log(2|\Phi|/\delta) + K\log{|\Upsilon|}}{N_S}.
\end{align*}

By our choice of $\gamma$, by the union bound with probability at least $1 - \delta$, the total variance variance gets upper bounded by $\frac{\epsilon^2}{2K\alpha_{\mathrm{max}^2}}$.

Note that the total bias is $2 \nu d K$. Since the sample complexity is decreasing in $\nu$.  We want to find largest $\nu$ such that variance is larger than bias. Thus we equate upper bound on variance to bias to compute $\nu$. 

This gives the optimal $\nu = \frac{\epsilon^2}{4dK^2\alpha_{\max}^2}$.

Plugging the values of $\nu$ and $\gamma$, we get
\begin{align*}
    N_S = \epsilon^{-4d}\bigg(2\alpha_{\mathrm{max}}\bigg)^{4d}\bigg((\log(2|\Phi|/\delta) + K\log{|\Upsilon|})\bigg)^d c^{-d}\psi^{-d}(6\sqrt{d})^d \cdot \left( d\log (6\sqrt{d} / \delta)\right) d^{d}K^{3d}.
\end{align*}
We can write the error bound $\epsilon$ in terms of $N_S$, we get:
\begin{align*}
   \epsilon = C_2 {N_S}^{-1/4d}\alpha_{\mathrm{\max}}\bigg((\log(2|\Phi|/\delta) + K\log{|\Upsilon|})\bigg)^{1/4}K^{3/4}d^{1/2}\psi^{-1/4},
\end{align*}
where $C_2$ is a positive constant (upto factor in $\log (6\sqrt{d} / \delta)^{1/4d}$).
\end{proof}
\begin{lemma}\label{lem:target_error_good_coverage}
    Suppose $\bar{\pi}$ is a policy satisfying Assumption~\ref{ass:targetwellexplored} and $n$ i.i.d. trajectories are sampled from the target task by policy $\bar{\pi}$. Then with probability at least $1 - \delta/2$ we can show that:
    \begin{align*}
        2 H \beta \sum_{h=1}^H \mathbb{E}_{(s_h, a_h) \sim \pi^\ast, P_h^\ast} [\|\widehat{\phi}_h (s_h, a_h)\|_{\Lambda_h} | s_1 = s] \leq c_1 n^{-1/2}d^{3/2}H^2\sqrt{\log(4dHn/\delta)},
    \end{align*}
for some finite constant $c_1$. 
\end{lemma}
\begin{proof}
    Proof directly follows from proof of Corollary 4.6 in \citep{jin2021pessimism}.
\end{proof}
\samplecomplexity*
\begin{proof}
    Let consider 2 events as the ones stated in Lemma~\ref{lem:uniform_source_error} and Lemma~\ref{lem:target_error_good_coverage}, each likely to happen with probability at least $1 - \delta/2$. By using union bound it is easy to show that both these events hold simultaneously with probability at least $1 - \delta$. Plugging these upper bounds in Theorem~\ref{thm:suboptbound} we get with probability at least $1 - \delta$:
    \begin{align*}
        &\textrm{SubOpt}(\widehat{\pi}, s) \leq 2 H \sum_{h=1}^H \mathbb{E}_{(s_h, a_h) \sim \pi^\ast, P_h^\ast}\bigg[ \underbrace{\epsilon(s_h, a_h)}_{\text{\rm{source coverage on $\pi^\ast$}}} +  (\beta + \underbrace{\epsilon_h}_{\text{\rm{source coverage on target}}}) \cdot \underbrace{\|\widehat{\phi}_h (s_h, a_h)\|_{\Lambda_h}}_{\text{\rm{target coverage on $\pi^\ast$}}}|s_1 = s\bigg]\\
        &\leq C_1' H^2 {N_S}^{-1/4d}\alpha_{\mathrm{\max}}\bigg((\log(2|\Phi|/\delta) + K\log{|\Upsilon|})\bigg)^{1/4}K^{3/4}d^{1/2}\psi^{-1/4}k \\
        & + C_2' n^{-1/2}d^{3/2}H^2\sqrt{\log(4dHn/\delta)}\\
        & + C_3 ' H^2 {N_S}^{-1/4d}\alpha_{\mathrm{\max}}\bigg((\log(2|\Phi|/\delta) + K\log{|\Upsilon|})\bigg)^{1/4}K^{3/4}d\psi^{-1/4}k\\
        & \leq \Tilde{\mathcal{O}}({N_S}^{-\frac{1}{4d}}n^{-\frac{1}{2}}H^2d^{\frac{3}{2}}K^{\frac{3}{4}}\sqrt{\log(1/\delta)}).
    \end{align*}
\end{proof}
\section{Experiment Details}\label{appendix:experiments}

\textbf{Environment Description:} In this section, we introduce the Combination lock (Comblock) environment, a widely adopted benchmark for algorithms designed for Block Markov Decision Processes (MDPs). Figure~\ref{fig:comblock} provides a visualization of the Comblock environment. Specifically, the environment encompasses a horizon denoted as H, and at each timestep h, it includes 3 latent states $z_{i;h}$, where $i \in \{0, 1, 2\}$, along with 5 possible actions. Within the three latent states, we designate $z_0$ and $z_1$ as the desirable states leading to the final reward, while $z_2$ represents undesirable states. At the onset of the task, the environment uniformly and independently samples one out of 5 possible actions for each good state $z_{0;h}$ and $z_{1;h}$ at each timestep h. These sampled actions, denoted as $a_{0;h}$ and $a_{1;h}$, respectively, are considered optimal actions corresponding to each latent state. These optimal actions, in conjunction with the task itself, dictate the dynamics of the environment. At each good latent state $s_{0;h}$ or $s_{1;h}$, taking the correct action results in a transition to either good state at the next timestep (i.e., $s_{0;h+1}$, $s_{1;h+1}$) with equal probability. Conversely, if the agent chooses any of the four bad actions, the environment deterministically transitions to the bad state $s_{2;h+1}$, and the bad states transition only to bad states at the subsequent timestep. The agent receives a reward in two scenarios: firstly, upon reaching the good states at the last timestep, the agent receives a reward of 1; secondly, upon the first transition into the bad state, the agent receives an "anti-shaped" reward of 0.1 with a probability of 0.5. This design renders greedy algorithms, lacking strategic exploration such as policy optimization methods, susceptible to failure. Regarding the initial state distribution, the environment begins in either $s_{0;0}$ or $s_{1;0}$ with equal probability. The dimension of the observation is $2^{\log(H+|S|+1)}$. For the emission distribution, given a latent state $s_{i;h}$, the observation is generated by concatenating the one-hot vectors of the state and the horizon. Additionally, i.i.d. $\mathcal{N}(0,0.1)$ noise is added at each entry, and if necessary, a 0 is appended at the end. Finally, a linear transformation is applied to the observation using a Hadamard matrix. It's noteworthy that, without effective features or strategic exploration, it requires $5^H$ actions with random actions to reach the final goal.


\textbf{Generating Source and Target Tasks:} To create the source environment, we randomly generate five instances of the Comblock environment as described. It's important to note that this approach ensures a shared emission distribution across the sources, while the latent dynamics differ due to independently and randomly selected optimal actions.
To construct the target environment, for each timestep h, we randomly select optimal actions at h from one of the sources and designate them as the optimal actions for the target environment at timestep h. This is contingent upon the condition that the selected optimal actions differ for the two good states. If the optimal actions are the same, we continue sampling until distinct actions are obtained. This procedure ensures variability in the optimal actions, introducing diversity in the latent dynamics of the target environment.
\end{document}